\documentclass{article}

\PassOptionsToPackage{round, compress}{natbib}

\usepackage[preprint]{neurips_2026}

\usepackage[utf8]{inputenc} %
\usepackage[T1]{fontenc}    %
\usepackage{hyperref}       %
\usepackage{url}            %
\usepackage{booktabs}       %
\usepackage{amsfonts}       %
\usepackage{nicefrac}       %
\usepackage{microtype}      %
\usepackage{xcolor}         %

\usepackage{xspace}
\usepackage{amsmath,amssymb} %
\usepackage{enumitem}

\usepackage{microtype}      %
\usepackage{multirow}
\usepackage[flushmargin]{footmisc}
\usepackage{pgf}
\usepackage{mathtools}
\usepackage{bbm}
\usepackage{amsthm}

\theoremstyle{plain}
\newtheorem{thm}{Theorem}[section]
\newtheorem{lemma}{Lemma}[section]
\newtheorem{asmp}{Assupmtion}
\newtheorem{definition}[thm]{Definition}
\newtheorem{theorem}[thm]{Theorem}

\newtheorem{corollary}[thm]{Corollary}

\usepackage{algorithm}
\usepackage{algpseudocode}
\usepackage{algorithmicx}

\usepackage{wrapfig}
\usepackage[flushleft]{threeparttable}
\usepackage{multirow}
\usepackage{subcaption}

\usepackage{tikz}
\usetikzlibrary{arrows.meta,calc}
\usetikzlibrary{bending} %
\tikzset{>=latex} %

\usepackage{pgfplots}
\usepgfplotslibrary{colormaps}
\pgfplotsset{compat=1.18}

\usepackage{tabularx}
\newcolumntype{Y}{>{\raggedright\arraybackslash}X}

\title{Accelerating LMO-Based Optimization via \\Implicit Gradient Transport}

\author{%
  Won-Jun Jang\qquad Si-Hyeon Lee \\
  School of Electrical Engineering\\
  Korea Advanced Institute of Science and Technology (KAIST)\\
  Daejeon, South Korea \\
  \{\texttt{wonjun}\_\texttt{jang}, \texttt{sihyeon}\}\texttt{@kaist.ac.kr}
  \\
}

\begin{document}

\maketitle

\begin{abstract}
Recent optimizers such as Lion and Muon have demonstrated strong empirical performance by normalizing gradient momentum via linear minimization oracles (LMOs). While variance reduction has been explored to accelerate LMO-based methods, it typically incurs substantial computational overhead due to additional gradient evaluations. At the same time, the theoretical understanding of LMO-based methods remains fragmented across unconstrained and constrained formulations.
Motivated by these limitations, we propose \emph{LMO-IGT}, a new class of stochastic LMO-based methods leveraging implicit gradient transport (IGT). We further introduce a unified framework for stochastic LMO-based optimization together with a new stationarity measure, the \emph{regularized support function} (RSF), which bridges gradient-norm and Frank--Wolfe-gap notions within a common framework. By evaluating stochastic gradients at transported points, LMO-IGT accelerates convergence while retaining the single-gradient-per-iteration structure of standard stochastic LMO.
Our analysis establishes that stochastic LMO achieves an iteration complexity of $\mathcal{O}(\varepsilon^{-4})$, variance-reduced LMO achieves $\mathcal{O}(\varepsilon^{-3})$ at the cost of additional gradient evaluations, and LMO-IGT achieves $\mathcal{O}(\varepsilon^{-3.5})$ using only a single stochastic gradient per iteration. Empirically, LMO-IGT consistently improves over stochastic LMO counterparts with negligible overhead. Among its instantiations, Muon-IGT achieves the strongest overall performance across evaluated settings, demonstrating that IGT provides an effective and practical acceleration mechanism for modern LMO-based optimization.
\end{abstract}

\section{Introduction}
\label{sec:introduction}
Deep learning has achieved remarkable success in recent years, increasing the need for reliable optimization methods for training large-scale neural networks. Training such models involves high-dimensional and non-convex optimization problems, making efficient and stable algorithms essential. Consequently, efficient stochastic gradient-based optimizers such as Adam~\citep{kingma2017adammethodstochasticoptimization} and AdamW~\citep{loshchilov2019decoupledweightdecayregularization} have long been the standard choice. 

Recently, new optimizers such as Lion~\citep{chen2023symbolic} and Muon~\citep{jordan2024muon} have attracted significant attention due to their strong empirical performance. A common feature of these methods is that they \textit{normalize} the gradient momentum during optimization, a property closely related to the linear minimization oracle (LMO). Following this, a line of work has focused on improving the convergence of LMO-based methods, where techniques such as variance reduction achieve faster theoretical rates~\citep{jiang2025convergence, pmlr-v267-yuan25f, chang2025convergencemuon, sfyraki2026lionsmuonsoptimizationstochastic}. However, these approaches typically require additional gradient evaluations, which can be computationally expensive in large-scale settings. This raises a natural question: can LMO-based methods achieve accelerated convergence {while avoiding the substantial computational overhead associated with variance reduction?}

From a theoretical perspective, LMO originates from the Frank--Wolfe method for constrained optimization~\citep{jiang2025convergence, sfyraki2026lionsmuonsoptimizationstochastic, pmlr-v267-pethick25a}. However, the role of LMO-type updates in modern stochastic optimization is not yet fully understood. In particular, existing analyses often rely on different stationarity measures for unconstrained and constrained problems, preventing a unified treatment across settings. This leads to a second question: can stochastic LMO-based optimization be understood within a unified framework?

We address the above questions through the following contributions.

\textbf{Low-overhead acceleration for LMO-based optimization.} To accelerate LMO-based methods while avoiding the substantial computational overhead of variance reduction, we draw inspiration from implicit gradient transport (IGT), which constructs momentum using gradients evaluated at a lookahead point~\citep{NEURIPS2019_1dba5eed, cutkosky2020momentum}. While IGT has been studied in normalized gradient methods, we extend it to LMO-based optimization. We observe that LMO updates inherently perform a form of momentum normalization, making them well-suited for IGT-style acceleration. Based on this insight, we propose a new class of methods, termed LMO-IGT, that achieves accelerated convergence without additional gradient evaluations.

\textbf{Unified framework and analysis.}
We develop a unified framework for stochastic LMO-based optimization, encompassing three classes: stochastic LMO, variance-reduced LMO (LMO-VR), and the proposed LMO-IGT. To enable a unified analysis across both unconstrained and constrained settings, we introduce a new stationarity measure, the \emph{regularized support function (RSF)}, which bridges gradient-based and Frank--Wolfe-type measures. Under this framework, we establish unified convergence results: stochastic LMO achieves an iteration complexity of $O(\epsilon^{-4})$, LMO-VR achieves $O(\epsilon^{-3})$, and LMO-IGT attains  $O(\epsilon^{-3.5})$, while incurring only modest computational overhead.

\textbf{Empirical validation.} We instantiate LMO-IGT with different LMO sets, yielding practical algorithms such as Lion-IGT and Muon-IGT. In particular, Muon-IGT achieves the strongest empirical performance without requiring extra gradient evaluations, demonstrating the effectiveness of the proposed approach in large-scale settings.

Table~\ref{tab:lmos} summarizes the three classes in our unified framework, together with representative algorithms, the associated LMO sets, and their corresponding iteration complexities.

\begin{table}[!htb]
     \small
    \caption{Unified framework for stochastic LMO-based optimization and the corresponding convergence results. For a gradient momentum matrix $g$, we write its singular value decomposition (SVD) as  $g=U\text{diag}(\sigma)V^\top$. See Section~II for the definition of the LMO and the set \(\mathcal{C}\).
}
    \begin{tabular}{ccccc}
    \toprule
        \multirow{2}{*}{Class} & \multirow{2}{*}{Algorithm} & \multirow{2}{*}{LMO set $\mathcal{C}$} &  \multirow{2}{*}{LMO$_\mathcal{C}(g)$}  & Iteration \\
          & & &  & Complexity \\
         \midrule\midrule

    \multirow{6.5}{*}{$\begin{array}{c}
         \text{Stochastic}\\
         \text{LMO}\\ \text{\footnotesize(1 gradient eval.)}
    \end{array}$}     & Normalized SGD~\citep{hazan2015beyond} & $\|\cdot\|_2$--ball & $-g/||g||$ &  \multirow{6.5}{*}{$O(\epsilon^{-4})$} \\
    \cmidrule{3-4}
           & signSGD~\citep{pmlr-v80-bernstein18a}& \multirow{3}{*}{$\|\cdot\|_\infty$--ball} & \multirow{3}{*}{$-$sign$(g)$}   \\
           & Signum~\citep{pmlr-v80-bernstein18a}&  &     \\
           & Lion~\citep{chen2023symbolic} &  &      \\
    \cmidrule{3-4}
           & Muon~\citep{jordan2024muon} & $\|\cdot\|_\text{op}$--ball$^*$ & $-UV^\top$   \\
         \midrule
    \multirow{4}{*}{$\begin{array}{c}
         \text{LMO-VR}  \\ \text{(2 gradient eval.)}
    \end{array}$}     & Lion-VR~\citep{jiang2025convergence} & $\|\cdot\|_\infty$--ball & $-$sign$(g)$ &  \multirow{4}{*}{$O(\epsilon^{-3})$} \\
    \cmidrule{3-4}
           & MARS-Shampoo~\citep{pmlr-v267-yuan25f} & \multirow{3}{*}{$\|\cdot\|_\text{op}$--ball$^*$ } &\multirow{3}{*}{$-UV^\top$}   \\
           
           & \multirow{2}{*}{Muon-VR~\citep{sfyraki2026lionsmuonsoptimizationstochastic}} & &    \\ \\
         \midrule
   \multirow{3.8}{*}{$\begin{array}{c}
         \text{LMO-IGT} \\ \text{(1 gradient eval.)}
    \end{array}$}    & NIGT~\citep{cutkosky2020momentum} &$\|\cdot\|_2$--ball & $-g/||g||$  &  \multirow{3.8}{*}{$O(\epsilon^{-3.5})$}\\
   \cmidrule{3-4}
        &  \textbf{Lion-IGT (this work)} &  $\|\cdot\|_\infty$--ball & $-$sign$(g)$ &   \\         
   \cmidrule{3-4}
        &  \textbf{Muon-IGT (this work)} &  $\|\cdot\|_\text{op}$--ball$^*$ & $-UV^\top$ &   \\         
         \bottomrule
    \end{tabular}
    \\\\
    \hspace*{\fill}$^*$ operator norm ball for matrices: $\mathcal{C}=\{X\,|\,\sigma_1(X)\le1\}$.\\
\label{tab:lmos}
\end{table}

\section{Problem Setting and Preliminaries}
\label{sec:prelim}

Let $\xi$ denote a random variable capturing the stochasticity of the oracle, and let $f(\cdot;\xi)$ denote the corresponding sample-wise loss function. We consider the stochastic first-order minimization problem
\begin{align}
\min_{w\in\mathcal{P}} F(w),
\qquad
F(w):=\mathbb{E}_{\xi}[f(w;\xi)], \label{eq:problem}
\end{align}
where \(F:\mathbb{R}^d \to \mathbb{R}\) and $f$ are non-convex, differentiable functions. We assume that the feasible set \(\mathcal{P}\subseteq\mathbb{R}^d\) is convex and compact.

We make the following assumptions:
\begin{asmp}[\(L\)-smoothness] \label{asmp:l_smooth}
The function \(F\) is \(L\)-smooth:
\[
\|\nabla F(x)-\nabla F(y)\|
\le
L\|x-y\|,
\qquad
\forall x,y\in\mathbb R^d.
\]
\end{asmp}

\begin{asmp}[Bounded variance]\label{asmp:bound_var}
    The stochastic gradients have bounded variance:
\[
\mathbb E_{\xi}\!\left[\|\nabla f(w;\xi)-\nabla F(w)\|^2\right]
\le
\sigma^2,
\qquad
\forall w\in\mathbb R^d.
\]
\end{asmp}

\begin{asmp}[Averaged smoothness]\label{asmp:avg_sample_smth}
There exists \(L>0\) such that
\[
\mathbb E_{\xi}\!\left[\|\nabla f(x;\xi)-\nabla f(y;\xi)\|^2\right]
\le
L^2\|x-y\|^2, \qquad \forall x, y\in\mathbb{R}^d.
\]
\end{asmp}

\begin{asmp}[Second-order smoothness]\label{asmp:sec_smooth}
The Hessian of \(F\) is $\rho$-Lipschitz:
\[
\|\nabla^2F(x)-\nabla^2F(y)\|_{\mathrm{op}}
\le
\rho\|x-y\|,
\qquad
\forall x,y\in\mathbb R^d.
\]
\end{asmp}

Assumptions~\ref{asmp:l_smooth} and~\ref{asmp:bound_var} are standard in stochastic optimization. Assumption~\ref{asmp:avg_sample_smth} is commonly used in variance-reduced methods, while Assumption~\ref{asmp:sec_smooth} is a standard second-order regularity condition in analyses of IGT \citep{cutkosky2020momentum}.

In the rest of this section, we introduce the necessary preliminaries.
We begin by describing the update rule of LMO-based optimization methods, followed by their Frank--Wolfe interpretation and the associated stationarity measure, the Frank--Wolfe gap. We then summarize the convergence behavior of stochastic LMO methods and discuss variance reduction as an acceleration technique. Finally, we present IGT as an alternative acceleration mechanism. {A more detailed discussion of related work is provided in Appendix~\ref{app:related_work}.}

\subsection{LMO-Based Updates}
\label{subsec:lmo-based-updates}

To solve~\eqref{eq:problem}, both normalized SGD and LMO-based optimization methods (e.g., Lion and Muon) compute, given a gradient or momentum estimate \(g\), the linear minimization oracle (LMO)
\[
\operatorname{LMO}_{\mathcal{C}}(g)
\in
\arg\min_{v\in\mathcal{C}} \langle g, v \rangle.
\]
Here, \(\mathcal{C} \subseteq \mathbb{R}^d\) is a compact  convex set containing the origin.  Equivalently, the oracle returns the point in \(\mathcal{C}\) most aligned with the descent direction \(-g\). The specific form of \(\operatorname{LMO}_{\mathcal{C}}(\cdot)\) depends on \(\mathcal{C}\). For example, if \(\mathcal{C}\) is the \(\ell_2\) norm ball, then \(\operatorname{LMO}_{\mathcal{C}}(g) = -g/\|g\|\), which recovers normalized SGD. For matrix-valued \(g\), choosing \(\mathcal{C}\) as the spectral norm ball, i.e., \(\mathcal{C} = \{X : \|X\|_\text{op} \leq 1\}\), yields updates corresponding to Muon.

At each iteration \(t = 0,1,\dots,T-1\), the parameter \(w_t\) is updated via  
\begin{align}\label{eq:lmo_update}
    w_{t+1}
=
(1 - \lambda \eta_t) w_t
+
\eta_t v_t, \qquad v_t = \operatorname{LMO}_{\mathcal{C}}(g_t),
\end{align}
where \(\lambda \geq 0\) is the weight decay parameter and \(\eta_t > 0\) is the learning rate. When \(\lambda = 0\), the problem reduces to the  unconstrained case with $\mathcal{P}=\mathbb{R}^d$. When \(\lambda > 0\), the update admits a natural constrained optimization interpretation, which we discuss next.

\subsection{Frank--Wolfe Interpretation}
\label{subsec:fw-interpretation}
The Frank--Wolfe (conditional gradient) method is designed for constrained optimization over a convex and compact set \(\mathcal{P}\). At iteration \(t\), given \(w_t\), stochastic gradient \(g_t\), and stepsize \(\gamma_t \in (0,1)\), the update is
\[
w_{t+1}
=
(1-\gamma_t) w_t
+
\gamma_t \operatorname{LMO}_{\mathcal{P}}(g_t),
\]
which preserves feasibility by forming a convex combination in \(\mathcal{P}\).

When \(\lambda > 0\), the LMO-based update in~\eqref{eq:lmo_update} can be interpreted as a Frank--Wolfe update~\citep{chen2024lion,sfyraki2026lionsmuonsoptimizationstochastic}. Since \(\mathcal{C}\) is compact and convex, the scaled set \(\mathcal{P} := \lambda^{-1}\mathcal{C}\) is also compact and convex. Defining \(s_t := \lambda^{-1} v_t\), where \(v_t = \operatorname{LMO}_{\mathcal{C}}(g_t)\), yields \(s_t = \operatorname{LMO}_{\mathcal{P}}(g_t) \in \mathcal{P}\). The update can thus be rewritten as
\[
w_{t+1}
=
(1-\lambda \eta_t) w_t
+
\eta_t v_t
=
(1-\gamma_t) w_t
+
\gamma_t s_t,
\]
with \(\gamma_t = \lambda \eta_t\), which matches the Frank--Wolfe update.

A standard stationarity measure for constrained optimization over \(\mathcal{P}\) is the Frank--Wolfe gap~\citep{jaggi2013revisiting}:
\[
G_{\mathcal{P}}(w)
:=
\sup_{u\in\mathcal{P}}
\langle -\nabla F(w),\, u - w \rangle.
\]
The gap satisfies \(G_{\mathcal{P}}(w) = 0\) if and only if \(w\) is a first-order stationary point over \(\mathcal{P}\)~\citep{lacoste2016convergence}. Moreover, its convergence to zero is equivalent to satisfying the KKT conditions~\citep{pmlr-v235-xie24e,sfyraki2026lionsmuonsoptimizationstochastic}.

\subsection{Convergence Analysis}
\label{subsec:baseline-convergence}
From the perspective of stochastic nonconvex optimization in the unconstrained setting, several recent works have analyzed the convergence of Lion/Muon-type methods under standard smoothness assumptions~\citep{dong2024convergence,jiang2025convergence,shen2025convergenceanalysismuon,chang2025convergencemuon,kim2026convergence,nagashima2026improvedconvergenceratesmuon,sato2025convergenceboundcriticalbatch} and established the rate
\begin{align}
    \frac{1}{T}\sum_{t=0}^{T-1}\|\nabla F(w_t)\|
    \le
    O(T^{-1/4}).
\end{align}
This implies an iteration complexity of \(O(\epsilon^{-4})\) to achieve \(\epsilon\)-stationarity.

An analogous result holds for constrained optimization from the Frank--Wolfe perspective~\citep{chen2024lion,sfyraki2026lionsmuonsoptimizationstochastic,pmlr-v267-pethick25a}. Using the Frank--Wolfe gap as the stationarity measure, one obtains
\begin{align}\label{eq:lmo_vr_bound}
    \frac{1}{T}\sum_{t=0}^{T-1} G_{\mathcal{P}}(w_t)
    \le
    O(T^{-1/4}),
\end{align}
which again yields \(O(\epsilon^{-4})\) iteration complexity. Thus, stochastic LMO-based methods exhibit the same convergence rate under both unconstrained and constrained formulations.

\subsection{Variance Reduction for LMO-Based Methods}
\label{subsec:variance-reduction}

Variance-reduction is a powerful tool for accelerating stochastic optimization~\citep{gower2020variancereducedmethodsmachinelearning,zhou2020stochastic,NEURIPS2019_b8002139,pmlr-v267-yuan25f,NEURIPS2018_1543843a}. When applied to Lion/Muon-type methods, it improves the iteration complexity to \(O(\epsilon^{-3})\)~\citep{jiang2025convergence,chang2025convergencemuon,sfyraki2026lionsmuonsoptimizationstochastic}. The key idea is to incorporate a STORM-like correction term~\citep{NEURIPS2019_b8002139} into the momentum update, 
\[
\nabla f(w_t;\xi_t) - \nabla f(w_{t-1};\xi_t),
\]
which reduces the variance of the stochastic gradient estimator while preserving first-order information.
However, this approach incurs additional overhead. It requires storing the previous iterate \(w_{t-1}\), which increases memory usage. Moreover, the correction term involves computing gradients at both \(w_t\) and \(w_{t-1}\) using the same data batch, effectively doubling the per-iteration gradient cost. As a result, despite improved theoretical convergence rates, variance reduction may degrade practical efficiency in large-scale settings.

\subsection{Implicit Gradient Transport}
\label{subsec:igt}

A separate line of work accelerates normalized methods without variance reduction by evaluating stochastic gradients at an extrapolated point~\citep{NEURIPS2019_1dba5eed,cutkosky2020momentum}. The key idea is to implicitly transport the momentum along the optimization trajectory using a lookahead point instead of the current iterate. For momentum parameter \(\beta_t\), implicit gradient transport (IGT) evaluates the gradient at
\[
x_{t+1}
=
w_{t+1}+\frac{\beta_t}{1-\beta_t}(w_{t+1}-w_t),
\]
and updates the momentum as
\[
m_t=\beta_t m_{t-1}+(1-\beta_t)\nabla f(x_t;\xi_t),
\qquad
w_{t+1}=w_t-\eta_t m_t.
\]
This mechanism preserves single-sample gradient evaluation while implicitly transporting the momentum along the optimization trajectory.

Earlier work analyzed this mechanism under stronger structural assumptions~\citep{NEURIPS2019_1dba5eed}. In particular, under second-order smoothness (Assumption~\ref{asmp:sec_smooth}), \citet{cutkosky2020momentum} established an improved  \(O(\epsilon^{-3.5})\) iteration complexity for \(\ell_2\)-normalized stochastic methods. However, existing IGT analyses are restricted to Euclidean normalized updates and do not extend to  general LMO-based methods or the constrained optimization setting.

\section{Main Results}\label{sec:main_results}
In this section, we develop a unified framework for stochastic LMO-based optimization and present our main theoretical results. We begin by introducing a new class of LMO-based methods, LMO-IGT, which extends IGT to general LMO-based updates. We then present a unified framework equipped with a new stationarity measure, enabling a common treatment of existing LMO-based approaches. Finally, we establish unified convergence guarantees for three classes of methods within this framework.

\subsection{LMO-IGT}
\label{subsec:lmo-igt-intro}
\begin{wrapfigure}{r}{0.55\textwidth}\vspace{-22pt}
\begin{minipage}{0.55\textwidth}
\begin{algorithm}[H]
\caption{LMO-IGT}
\label{alg:igt-lmo}
\begin{algorithmic}[1]
\Require initial point \(x_{0}=w_0\), momentum buffer \(m_{-1} = \nabla f(w_0;\xi_0)\), parameters \(\{\beta_{1,t}\}_{t\ge 0}\), \(\{\beta_{2,t}\}_{t\ge 0}\), \(\lambda\ge 0\), step sizes \(\{\eta_{1,t}\}_{t\ge 0}\), \(\{\eta_{2,t}\}_{t\ge 0}\)
\For{\(t=0,\ldots,T-1\)}
    \State \(g_t \gets \beta_{1,t}m_{t-1}+(1-\beta_{1,t})\nabla f(x_t;\xi_t)\)
    \State \(m_t \gets \beta_{2,t}m_{t-1}+(1-\beta_{2,t})\nabla f(x_t;\xi_t)\)
    \State \(v_t \gets \operatorname{LMO}_{\mathcal C}(g_t)\)
    \State \(x_{t+1} \gets (1-\lambda\eta_{1,t})w_t+\eta_{1,t} v_t\)
    \State \(w_{t+1} \gets (1-\lambda\eta_{2,t})w_t+\eta_{2,t} v_t\)
\EndFor
\end{algorithmic}
\end{algorithm}
\end{minipage}\vspace{-15pt}
\end{wrapfigure}

Algorithm~\ref{alg:igt-lmo} presents LMO-IGT. The key idea is to evaluate the stochastic gradient at a transported point \(x_t\), which serves as a lookahead approximation of the optimization trajectory, while updating the actual iterate \(w_t\) through an LMO-based step.

At each iteration, the method computes a stochastic gradient at \(x_t\) and uses it to update two momentum variables with different time scales, controlled by \(\beta_{1,t}\) and \(\beta_{2,t}\). Throughout the paper, we consider the regime \(0 \le \beta_{1,t} \le \beta_{2,t} < 1\), which corresponds to a Lion-style double-momentum design. Under this condition, the gradient momentum \(g_t\), controlled by \(\beta_{1,t}\), places more weight on the current stochastic gradient and is therefore more responsive to the local geometry, while the momentum buffer \(m_t\), controlled by \(\beta_{2,t}\), evolves more slowly and provides a more stable estimate for subsequent iterations. The gradient momentum \(g_t\) is then used to determine the update direction via the linear minimization oracle. 
The next transported point \(x_{t+1}\) and the iterate \(w_{t+1}\) are updated using the same LMO direction but potentially different step sizes. This design enables implicit gradient transport without increasing the number of gradient evaluations, thereby preserving the single-gradient-per-iteration structure of stochastic LMO methods.

For the parameter choice used in the convergence analysis, we set $\eta_{1,t}=\frac{1}{1-\beta_{2,t}}\eta_{2,t},
\qquad
t=0,1,\ldots,T-1.$ Then the auxiliary point satisfies
\[
x_{t+1}
=
w_{t+1}+\frac{\beta_{2,t}}{1-\beta_{2,t}}(w_{t+1}-w_t)
=
w_t+\frac{1}{1-\beta_{2,t}}(w_{t+1}-w_t),
\]
which is the natural LMO analogue of the IGT extrapolation in Section~\ref{subsec:igt}.

Different choices of \(\mathcal C\) yield concrete instances of the method. In particular, choosing an \(\ell_\infty\)-ball yields Lion-IGT, while choosing an operator-norm ball yields Muon-IGT, as summarized in Table~\ref{tab:lmos}.

{To compare LMO-IGT with existing stochastic LMO methods under a common stationarity measure, we introduce a unified framework in the next subsection. Within this framework, LMO-IGT achieves a faster convergence rate than stochastic LMO. Figure~\ref{fig:mom_comp} provides intuition for this improvement. In stochastic LMO, the momentum buffer $m_t = \beta_{2,t} m_{t-1} + (1 - \beta_{2,t}) \nabla f(w_t; \xi_t)$ is an exponential moving average and can therefore lag behind the true gradient $\nabla F(w_t)$. As a result, the query constructed from $m_t$ may be misaligned with the current optimization trajectory. 
LMO-IGT mitigates this lag by evaluating the stochastic gradient at a transported point $x_t$, chosen along the same LMO direction as the iterate update. Through the same momentum recursion, this leads $m_t$ to better approximate $\nabla F(w_t)$, thereby producing a more informative query $g_t$ for the LMO--without requiring an additional gradient evaluation.}

\begin{figure}[t]
\centering

\begin{subfigure}{0.46\textwidth}
\centering
\includegraphics[width=\linewidth]{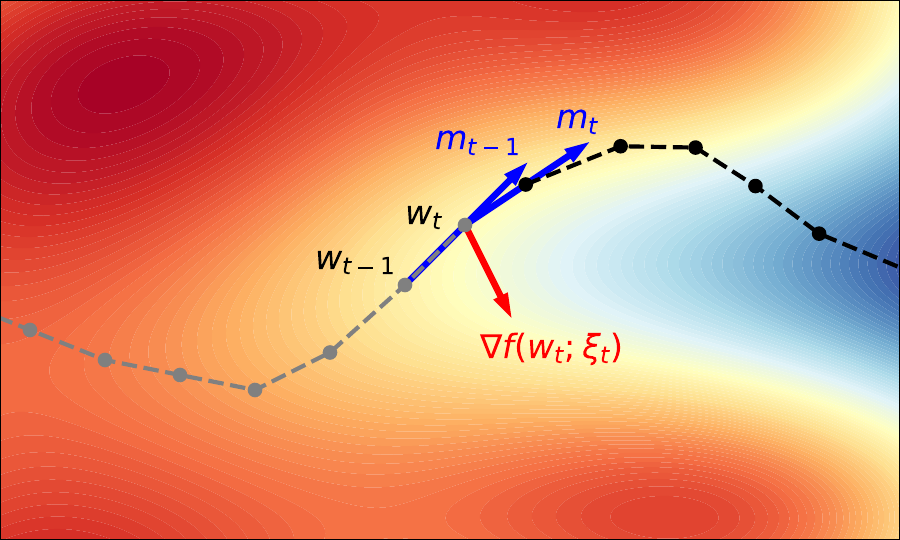}
\caption{Stochastic momentum}
\label{fig:lmo}
\end{subfigure}
\hfill
\begin{subfigure}{0.46\textwidth}
\centering
\includegraphics[width=\linewidth]{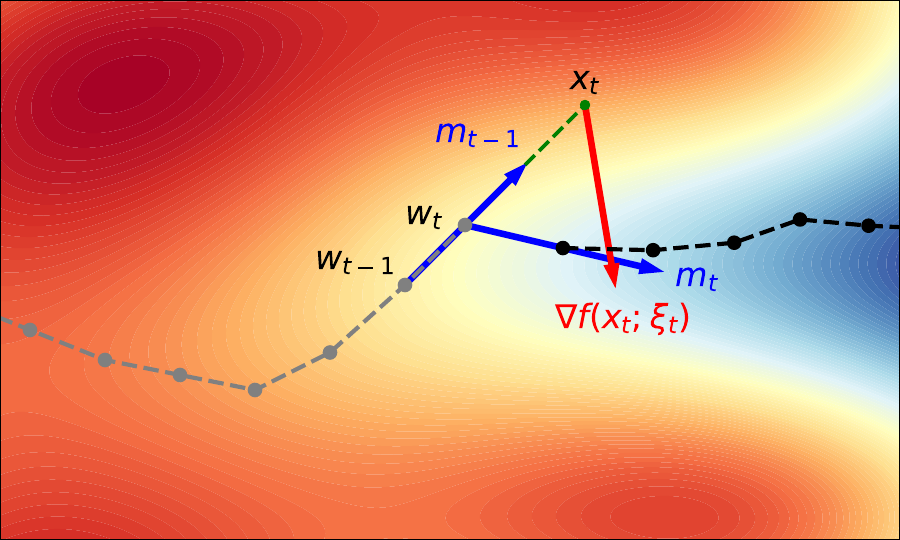}
\caption{IGT (transported gradient)}
\label{fig:lmo-igt}
\end{subfigure}

\caption{Comparison of stochastic momentum and IGT. In stochastic LMO, the momentum buffer is formed from gradients evaluated at the current iterate, which introduces a lagging effect. In LMO-IGT, the transported point \(x_t\) yields a momentum buffer that is better aligned with the optimization trajectory.}
\label{fig:mom_comp}
\end{figure}

\subsection{Unified Framework}
\label{subsec:regularized-support-framework}
To compare LMO-IGT with existing stochastic LMO methods, we introduce a unified framework along with a new stationarity measure.  
Let
\[
R:=\sup_{u,v\in\mathcal C}\|u-v\|
\]
denote the diameter of $\mathcal C$. In the constrained case with $\lambda>0$, the feasible region induced by the LMO update is $\mathcal{P}=\lambda^{-1}\mathcal{C}$, and the Frank--Wolfe gap $G_{\mathcal{P}}(w)$ serves as a natural stationarity measure. However, this quantity scales as $\lambda^{-1}R\|\nabla F(w)\|$ and thus  diverges as $\lambda\to 0$. To treat the constrained and unconstrained regimes in a unified manner, we introduce the following regularized support function.

\begin{definition}[Regularized support function]
\label{def:rsf}
For $w\in\mathbb R^d$ and $\lambda\ge 0$, define
\[
\Psi_{\mathcal C,\lambda}(w)
:=
\sup_{v\in\mathcal C}\langle -\nabla F(w),\,v-\lambda w\rangle.
\]
\end{definition}
When $\lambda>0$ and $\mathcal P=\lambda^{-1}\mathcal C$, this reduces to a scaled Frank--Wolfe gap:
\[
\Psi_{\mathcal C,\lambda}(w)=\lambda G_{\mathcal P}(w).
\]
When $\lambda=0$, it becomes the support function~\citep{fenchel1934theorie, gardner1995geometric, schneider2013convex} of $\mathcal C$ evaluated at $-\nabla F(w)$:
\[
\Psi_{\mathcal C,0}(w)=h_{\mathcal C}(-\nabla F(w))
:=\sup_{v\in\mathcal C}\langle -\nabla F(w),v\rangle.
\]
For centrally symmetric sets $\mathcal C$, including all norm balls considered in this paper, this coincides with the dual norm induced by $\mathcal C$. In particular, if $\mathcal C$ is a Euclidean ball with diameter $R$, then $\Psi_{\mathcal C,0}(w)=\frac{R}{2}\|\nabla F(w)\|.$
More generally, $\Psi_{\mathcal C,\lambda}(w)\le R\|\nabla F(w)\|.$ Thus, RSF interpolates naturally between gradient-norm stationarity and Frank--Wolfe-gap stationarity. Moreover, \(\Psi_{\mathcal C,\lambda}(w) = 0\) characterizes stationarity in both regimes: it is equivalent to the KKT conditions when \(\lambda > 0\), and to \(\nabla F(w) = 0\) when \(\lambda = 0\). %

To analyze existing LMO-based algorithms in a unified manner, we consider the following generalized update framework.

\begin{algorithm}[H]
\caption{Unified Framework for Stochastic LMO-Based Optimization}
\label{alg:unified-lmo-framework}
\begin{algorithmic}[1]
\Require initial point \(x_{-1}=x_{0}=w_0\), initial momentum \(m_{-1}=\nabla f(w_0;\xi_0)\), parameters \(\{\alpha_{1,t}\}_{t\ge 0}\), \(\{\alpha_{2,t}\}_{t\ge 0}\), \(\{\beta_{1,t}\}_{t\ge 0}\), \(\{\beta_{2,t}\}_{t\ge 0}\), \(\lambda\ge 0\), and step sizes \(\{\eta_{1,t}\}_{t\ge 0}\), \(\{\eta_{2,t}\}_{t\ge 0}\)
\For{\(t=0,\ldots,T-1\)}
    \State \(g_t \gets \beta_{1,t}m_{t-1}+(1-\beta_{1,t})\nabla f(x_t;\xi_t)+\alpha_{1,t}\bigl(\nabla f(x_t;\xi_t)-\nabla f(x_{t-1};\xi_t)\bigr)\)
    \State \(m_t \gets \beta_{2,t}m_{t-1}+(1-\beta_{2,t})\nabla f(x_t;\xi_t)+\alpha_{2,t}\bigl(\nabla f(x_t;\xi_t)-\nabla f(x_{t-1};\xi_t)\bigr)\)
    \State \(v_t \gets \operatorname{LMO}_{\mathcal C}(g_t)\)
    \State \(x_{t+1} \gets (1-\lambda\eta_{1,t})w_t+\eta_{1,t} v_t\)
    \State \(w_{t+1} \gets (1-\lambda\eta_{2,t})w_t+\eta_{2,t} v_t\)
\EndFor
\end{algorithmic}
\end{algorithm}

Algorithm~\ref{alg:unified-lmo-framework} presents a unified framework for stochastic LMO-based optimization, written in terms of stochastic gradients evaluated at the current point and explicitly exposing the shared LMO update structure. The parameters \(\alpha_{1,t}\) and \(\alpha_{2,t}\) capture variance-reduction effects through gradient differences, while \(\beta_{1,t}\) and \(\beta_{2,t}\) control the momentum dynamics.

{Stochastic LMO, LMO-VR, and the proposed LMO-IGT arise as special cases of this framework. In particular, setting $\alpha_{1,t}=\alpha_{2,t}=0$ and $\eta_{1,t}=\eta_{2,t}$ recovers stochastic LMO; allowing nonzero $\alpha_{1,t}$ and/or $\alpha_{2,t}$ with $\eta_{1,t}=\eta_{2,t}$ recovers LMO-VR; and setting $\alpha_{1,t}=\alpha_{2,t}=0$ with $\eta_{1,t}\neq \eta_{2,t}$ recovers LMO-IGT.}

\subsection{Unified Convergence Analysis}
\subsubsection{Stochastic LMO}
\label{subsec:vanilla-nesterov-results}

We begin with the stochastic LMO class. In Algorithm~\ref{alg:unified-lmo-framework}, this corresponds to setting \(\alpha_{1,t}=\alpha_{2,t}=0\) and \(\eta_{1,t}=\eta_{2,t}\) for all \(t=0,1,\ldots,T-1\), so that \(x_{t+1}=w_{t+1}\). %
The convergence result for stochastic LMO is given below.

\begin{thm}[Stochastic LMO]
\label{thm:vanilla-rsf}
Suppose Assumptions~\ref{asmp:l_smooth}--\ref{asmp:bound_var} hold. %
Assume that, for each $t=0,1,\ldots,T-1$, we set $\beta_{1,t}=\beta_1$, $\beta_{2,t}=\beta_2$, and $\eta_t=\eta$, where $0\le \beta_1\le \beta_2<1,$ $0\le \lambda\eta\le 1$ and $\Delta_F:=F(w_0)-F^*$ where $F^*:=\inf_wF(w)>-\infty$. Then
\begin{align}
\frac{1}{T}\sum_{t=0}^{T-1}\mathbb E\!\left[\Psi_{\mathcal C,\lambda}(w_t)\right]
&\le
\frac{\Delta_F}{T\eta}
+R\sigma\!\left(
(1-\beta_1)
+\beta_1\sqrt{1-\beta_2}
+\frac{\beta_1}{T(1-\beta_2)}+\frac1T
\right)
\nonumber\\
&\quad
+LR^2\eta\left(
\frac{\beta_1}{1-\beta_2}
+\frac12
\right).
\label{eq:vanilla-rsf}
\end{align}
In particular, choosing $\eta=\frac{1}{RT^{3/4}}, \beta_2=1-\frac{1}{T^{1/2}}, 1-\beta_1\in\left[\frac{1}{T^{1/2}},\frac{1}{T^{1/4}}\right]$ that {balances the convergence rate of each term in the RHS of \eqref{eq:vanilla-rsf}}
yields
$$
\frac{1}{T}\sum_{t=0}^{T-1}\mathbb E\!\left[\Psi_{\mathcal C,\lambda}(w_t)\right]
\le
O\!\left(
\frac{R(\Delta_F+\sigma+L)}{T^{1/4}}
\right)
=
O(T^{-1/4}).
$$
\end{thm}

The above result establishes the baseline RSF convergence guarantee for the stochastic-LMO class. {The proof of Theorem~\ref{thm:vanilla-rsf} is deferred to Appendix~\ref{app:vanilla}.}

\subsubsection{LMO-VR}
\label{subsec:vr-results}
We next consider the variance-reduced LMO class. This corresponds to a specialization of Algorithm~\ref{alg:unified-lmo-framework} that sets \(\eta_{1,t}=\eta_{2,t}\) for all \(t\), yielding \(x_{t+1}=w_{t+1}\), while activating the correction parameters \(\alpha_{1,t}\) and/or \(\alpha_{2,t}\). %

The convergence result for this class is given below.
\begin{thm}[LMO-VR]
\label{thm:gvr-main}
Suppose Assumptions~\ref{asmp:l_smooth}--\ref{asmp:avg_sample_smth} hold. %
Assume that, for each $t=0,1,\ldots,T-1$, we set $\alpha_{1,t}=\alpha_1$, $\alpha_{2,t}=\alpha_2$, $\beta_{1,t}=\beta_1$, $\beta_{2,t}=\beta_2$, and $\eta_t=\eta$, where $0\le \beta_1\le \beta_2<1$ and $0\le \lambda\eta\le 1$. Let $F^*$ be any finite lower bound on $F$, and $\Delta_F:=F(w_0)-F^*.$ Then
\begin{align}
\frac{1}{T}\sum_{t=0}^{T-1}\mathbb E\big[\Psi_{\mathcal C,\lambda}(w_t)\big]
&\le
\frac{\Delta_F}{T\eta}
+R\sigma\left(
(1-\beta_1)
+\beta_1\sqrt{1-\beta_2}
+\frac{\beta_1}{T(1-\beta_2)}
+\frac{1}{T}
\right)
\nonumber\\
&\quad
+LR^2\eta\left(
|\beta_1-\alpha_1|
+\frac{\beta_1|\beta_2-\alpha_2|}{1-\beta_2}
+|\alpha_1|
+\frac{\beta_1|\alpha_2|}{\sqrt{1-\beta_2}}
+\frac12
\right).
\label{eq:gvr-main}
\end{align}
\end{thm}

\begin{corollary}
\label{cor:gvr-stochastic}
Under the assumptions of Theorem~\ref{thm:gvr-main}, suppose $0\le\alpha_1\le\beta_1, \alpha_2=\beta_2.$ Choose $\eta=\frac{1}{RT^{2/3}}, \beta_2=1-\frac{1}{T^{2/3}}, 1-\beta_1\in\left[\frac{1}{T^{2/3}},\frac{1}{T^{1/3}}\right]$ {that balances the convergence rate of each term in the RHS of \eqref{eq:gvr-main}.} Then
\[
\frac{1}{T}\sum_{t=0}^{T-1}\mathbb E\big[\Psi_{\mathcal C,\lambda}(w_t)\big]
\le
O\left(\frac{R(\Delta_F+\sigma+L)}{T^{1/3}}\right)=
O\!\left(T^{-1/3}\right).
\]
\end{corollary}

 {The proofs of Theorem~\ref{thm:gvr-main} and Corollary~\ref{cor:gvr-stochastic} are deferred to Appendix~\ref{app:gvr}.} The result shows that variance reduction improves the convergence rate from $O(T^{-1/4})$ to $O(T^{-1/3})$. This improvement arises from explicitly controlling the noise in the gradient momentum through the correction terms. However, this approach requires an additional stochastic gradient evaluation per iteration, trading increased per-iteration cost for a faster convergence rate. 

\subsubsection{LMO-IGT}
\label{subsec:igt-results}

{The LMO-IGT method in Algorithm~\ref{alg:igt-lmo} corresponds to a specialization of Algorithm~\ref{alg:unified-lmo-framework}, setting \(\alpha_{1,t}=\alpha_{2,t}=0\) for all $t=0, 1, ..., T-1$.} Compared with LMO-VR, the main advantage of LMO-IGT is that it preserves the single-gradient-per-iteration structure of stochastic LMO, while exploiting the second-order regularity in Assumption~\ref{asmp:sec_smooth} through the transported point \(x_t\). The convergence result for LMO-IGT is given below.

\begin{thm}[LMO-IGT]
\label{thm:igt-rsf}
Suppose Assumptions~\ref{asmp:l_smooth},~\ref{asmp:bound_var}, and~\ref{asmp:sec_smooth} hold. Consider Algorithm~\ref{alg:igt-lmo} and assume that, for each $t=0,1,\ldots,T-1$, we set $\beta_{1,t}=\beta_1$, $\beta_{2,t}=\beta_2$, $\eta_{2,t}=\eta$, and $\eta_{1,t}=\eta/(1-\beta_2)$, where $0\le \beta_1\le \beta_2<1,$ $0\le \lambda\eta\le 1$ and $\Delta_F:=F(w_0)-F^*$. Then
\begin{align}
\frac{1}{T}\sum_{t=0}^{T-1}
\mathbb E\!\left[\Psi_{\mathcal C,\lambda}(w_t)\right]
&\le
\frac{\Delta_F}{T\eta}
+R\sigma\left(
(1-\beta_1)
+\beta_1\sqrt{1-\beta_2}
+\frac{\beta_1}{T(1-\beta_2)}+\frac1T
\right)
\nonumber\\
&\quad
+LR^2\eta\left(
\frac{\beta_2-\beta_1}{1-\beta_2}
+\frac12
\right)
+\rho R^3\eta^2\left(
\frac{\beta_1}{1-\beta_2}
+\frac{\beta_2^2}{(1-\beta_2)^2}
\right).
\label{eq:igt-rsf}
\end{align}
\end{thm}

\begin{corollary}[Stochastic regime]
\label{cor:igt-stochastic}
Under the assumptions of Theorem~\ref{thm:igt-rsf}, suppose $\sigma>0$. By choosing $\eta=\frac{1}{RT^{5/7}}, \beta_2=1-\frac{1}{T^{4/7}}, 1-\beta_1\in\left[\frac{1}{T^{4/7}},\frac{1}{T^{2/7}}\right]$  that balances the convergence rate of each term in the RHS of \eqref{eq:igt-rsf}, Algorithm~\ref{alg:igt-lmo} yields
\[
\frac{1}{T}\sum_{t=0}^{T-1}\mathbb E\!\left[\Psi_{\mathcal C,\lambda}(w_t)\right]
\le
O\!\left(
\frac{R(\Delta_F+\sigma+\rho)}{T^{2/7}}
+\frac{RL}{T^{5/7}}
\right)
=
O(T^{-2/7}).
\]
\end{corollary}

\begin{corollary}[Deterministic regime]
\label{cor:igt-deterministic}
Under the assumptions of Theorem~\ref{thm:igt-rsf}, suppose $\sigma=0$. By choosing $\eta=\frac{1}{RT^{1/2}}, \beta_1=\beta_2=1-\frac{1}{T^{1/4}}$  that balances the convergence rate of each term in the RHS of \eqref{eq:igt-rsf}, Algorithm~\ref{alg:igt-lmo} yields
\[
\frac{1}{T}\sum_{t=0}^{T-1}\mathbb E\!\left[\Psi_{\mathcal C,\lambda}(w_t)\right]
\le
O\!\left(
\frac{R(\Delta_F+L+\rho)}{T^{1/2}}
\right)
=
O(T^{-1/2}).
\]
\end{corollary}

The proofs of Theorem~\ref{thm:igt-rsf} and Corollaries~\ref{cor:igt-stochastic} and~\ref{cor:igt-deterministic} are deferred to Appendix~\ref{app:igt}.

Taken together, the three classes exhibit the hierarchy
\[
O(T^{-1/4})
\quad\text{(stochastic LMO)},\qquad
O(T^{-1/3})
\quad\text{(LMO-VR)},\qquad
O(T^{-2/7})
\quad\text{(LMO-IGT)}.
\]
Equivalently, their iteration complexities are \(O(\epsilon^{-4})\), \(O(\epsilon^{-3})\), and \(O(\epsilon^{-3.5})\), respectively. The key computational distinction is that LMO-VR achieves its improvement through a correction term with substantial per-iteration cost, whereas LMO-IGT attains faster convergence through a low-overhead extrapolation.

\section{Experiments}
\label{sec:experiments}

We evaluate the proposed IGT method on CIFAR-10~\citep{krizhevsky2009learning} (MIT License) for image classification using ResNet-18~\citep{he2016deep}. We compare eight optimizers: AdamW, NIGT, Lion, Muon, Lion-VR, Muon-VR, Lion-IGT, and Muon-IGT. AdamW serves as a standard adaptive optimization baseline, while NIGT serves as an existing Euclidean normalized IGT baseline. The Lion and Muon methods belong to the stochastic LMO class, Lion-VR and Muon-VR correspond to the LMO-VR class, and Lion-IGT and Muon-IGT instantiate the proposed LMO-IGT class. 
{Experimental details are given in Appendix~\ref{app:experimental-details}.}  %

\begin{figure}[!htb]
    \centering
    \begin{subfigure}[b]{0.49\textwidth}
        \includegraphics[width=\linewidth]{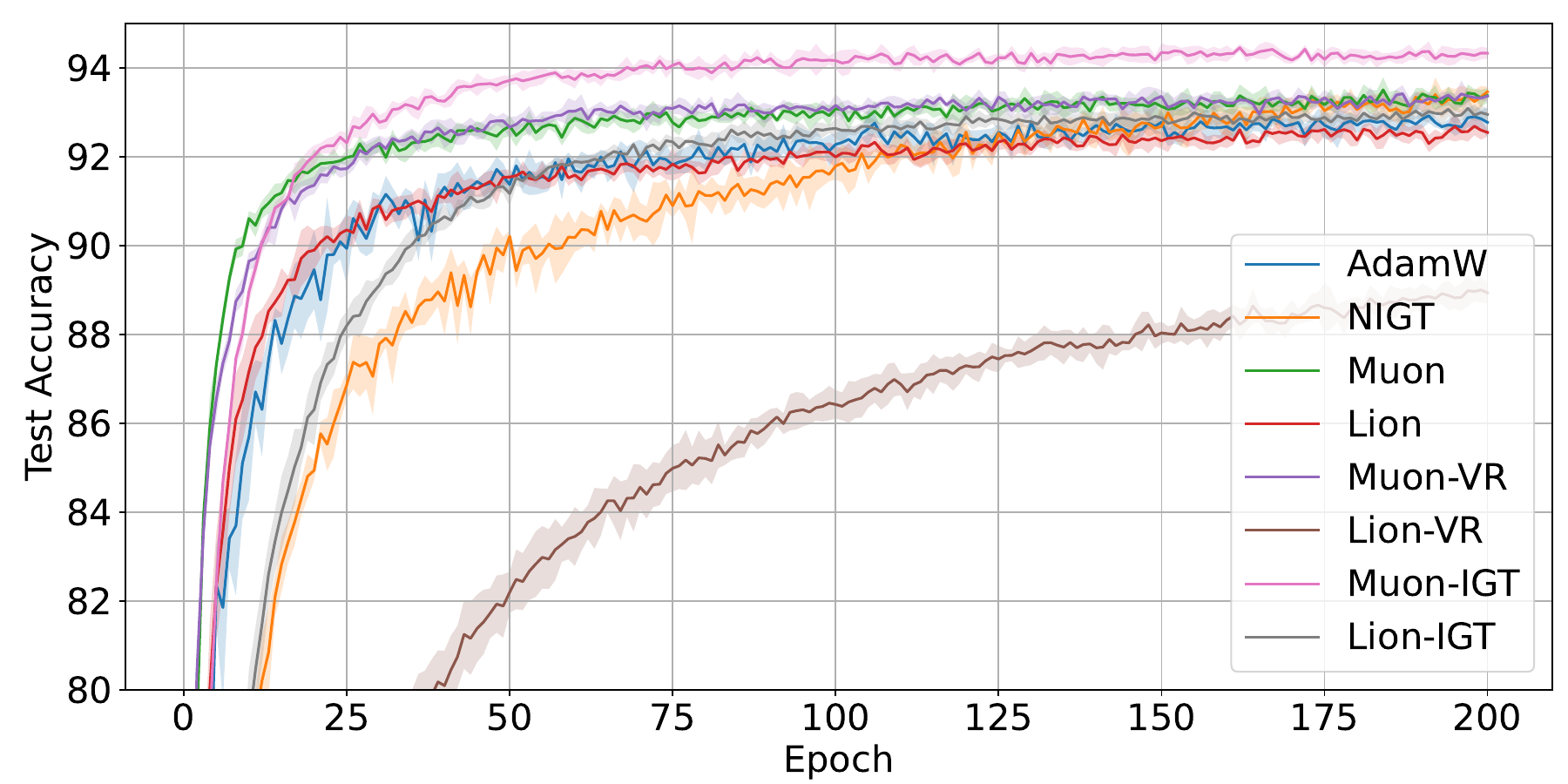}
        \caption{Test accuracy over epochs.}
        \label{fig:main_results_epoch}
    \end{subfigure}
    \begin{subfigure}[b]{0.49\textwidth}
        \includegraphics[width=\linewidth]{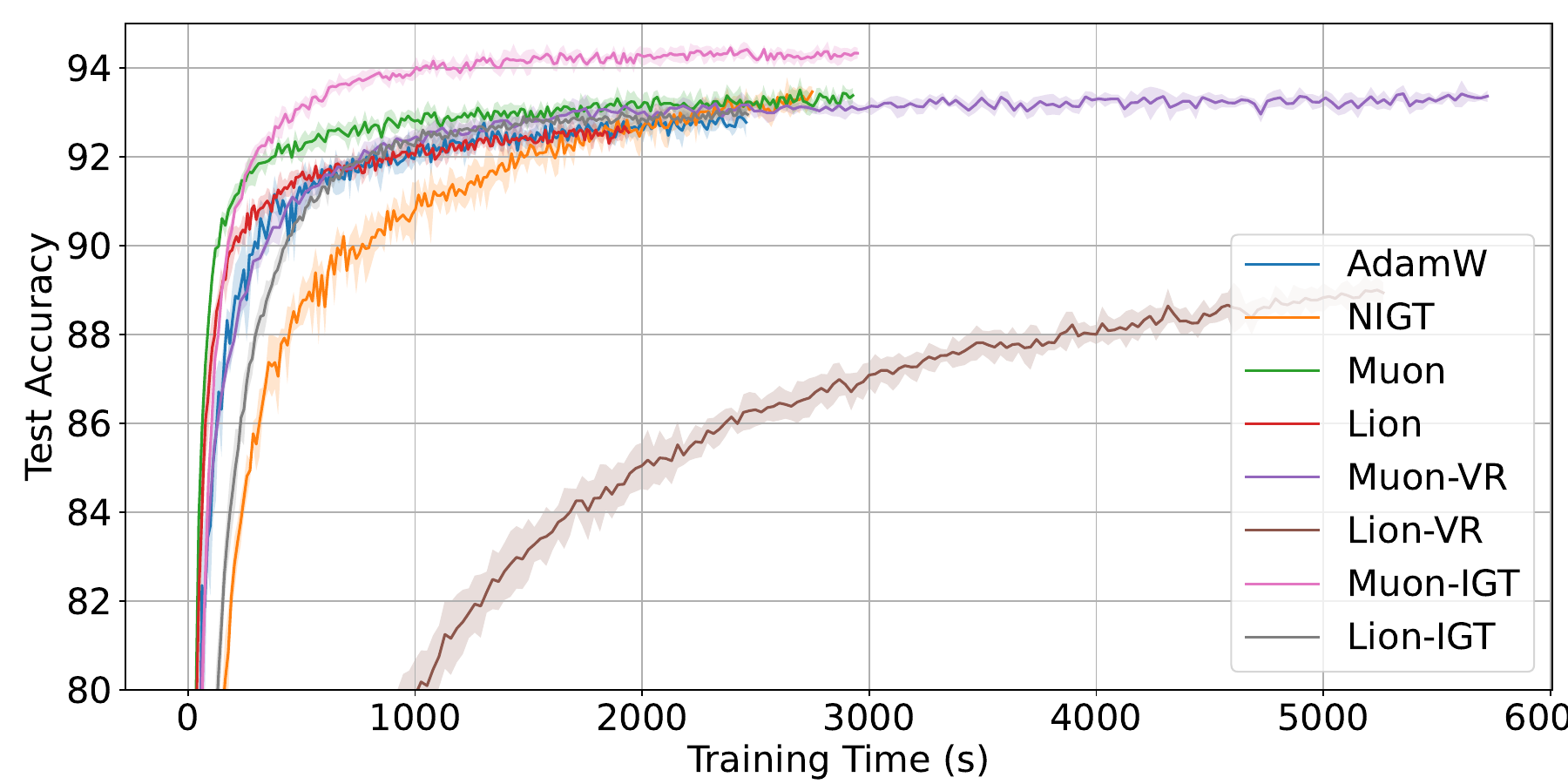}
        \caption{Test accuracy over training time.}
        \label{fig:main_results_time}
    \end{subfigure}
    \caption{Test accuracy on CIFAR-10 with ResNet-18. Each curve is averaged over five independent runs, and the shaded region denotes one standard deviation. All methods are trained for 200 epochs. In Figure~\ref{fig:main_results_time}, curves terminate at different time points as faster methods complete training earlier.}
    \label{fig:main_results} 
\end{figure}

{In Figure~\ref{fig:main_results}, each curve reports the mean test accuracy over five independent runs and the shaded region indicates one standard deviation.} 
Figure~\ref{fig:main_results_epoch} shows that Muon-IGT achieves the best final accuracy among all tested methods. Muon and NIGT form the strongest baselines, indicating that both operator-norm LMO geometry and Euclidean IGT are effective; nevertheless, Muon-IGT improves upon both by combining implicit gradient transport with Muon-style LMO updates. AdamW remains competitive but is outperformed by the strongest LMO-based methods.

The comparison with LMO-VR highlights the practical advantage of IGT. Although variance reduction improves the theoretical convergence rate, its correction term requires an additional gradient evaluation at the previous iterate, substantially increasing wall-clock time as shown in Figure~\ref{fig:main_results_time}. In contrast, IGT uses only one stochastic gradient evaluation per iteration and incurs minimal computation and memory overhead. Consequently, Lion-IGT and Muon-IGT run at nearly the same speed as their plain counterparts while achieving higher accuracy. We also observe that Lion-VR performs poorly in this setting, suggesting that  variance-reduction corrections are not uniformly beneficial for sign-based LMO updates.

Additional results, including ablation studies on the components of Muon-IGT, robustness tests over learning rate and weight decay, and experiments on other tasks such as language modeling and large-scale language modeling, are provided in Appendix~\ref{app:additional-experiments}.  These results demonstrate the robustness of Muon-IGT to hyperparameter choices and its strong performance across different tasks.

\section{Conclusion}
\label{sec:conclusion}
We studied stochastic LMO-based optimization and proposed LMO-IGT, a new class of methods that extends IGT to general LMO-based updates. We also introduced a unified framework equipped with a new stationarity measure, enabling a common analysis of stochastic LMO, LMO-VR, and LMO-IGT methods across both constrained and unconstrained settings.

Our analysis shows that LMO-IGT improves the convergence rate from \(O(\epsilon^{-4})\) to \(O(\epsilon^{-3.5})\) while preserving the single-gradient-per-iteration structure of stochastic LMO, thereby bridging the gap between standard stochastic LMO and variance-reduced methods. Empirically, LMO-IGT consistently outperforms its counterparts while maintaining comparable wall-clock efficiency, highlighting its practical advantage. 
These results demonstrate that IGT provides an effective mechanism for accelerating LMO-based optimization without sacrificing computational efficiency.

{Limitations and broader impacts are provided in Appendix~\ref{sec:limitations} and Appendix~\ref{sec:broader-impact}, respectively.}

\bibliography{main}
\bibliographystyle{plainnat}

\newpage
\appendix
\setcounter{equation}{0}
\renewcommand{\theequation}{A\arabic{equation}}

\section{Related Works}

\label{app:related_work}

Viewed historically, the literature leading to our setting has evolved along three largely independent directions: improved local scaling via adaptive moments, improved geometry via normalization, sign-based, or matrix-aware updates, and improved stochastic efficiency via recursive estimators or extrapolation-based mechanisms, including lookahead and transported gradients. A central challenge in comparing these lines is that they adopt different notions of stationarity. Some works measure convergence using (averaged) gradient norms, others rely on dual norms induced by normalized updates, while projection-free methods naturally employ the Frank--Wolfe gap. As a result, even algorithmically similar methods are often analyzed under incompatible criteria. Given this mismatch, the literature is best understood from a conceptual perspective before drawing comparisons based on formal convergence guarantees.

\paragraph{Adaptive methods and coordinate-wise preconditioning.}
Coordinate-wise adaptive preconditioning has a long history in stochastic optimization. AdaGrad adapts per-coordinate learning rates using accumulated gradient geometry \citep{duchi2011adaptive}, while RMSProp replaces this cumulative scaling with an exponential moving average of recent squared gradients \citep{tieleman2012rmsprop}. The modern baseline in deep learning is Adam, which combines momentum with coordinate-wise second-moment normalization through exponential moving averages of first and second moments \citep{kingma2017adammethodstochasticoptimization}. Subsequent work has clarified that adaptive methods are not merely drop-in improvements over SGD. In particular, \citet{reddi2018convergence} identified non-convergence issues of Adam and proposed AMSGrad, while later studies analyzed broader Adam-type recursions and Yogi-style stabilizations in nonconvex stochastic settings \citep{chen2019convergence, NEURIPS2018_90365351}. AdamW is especially relevant in our context, as decoupled weight decay separates shrinkage from the gradient update \citep{loshchilov2019decoupledweightdecayregularization}. Recent work further shows that this separation alters the effective optimization geometry and admits an implicit constrained interpretation connected to normalized steepest descent and Frank--Wolfe-type methods \citep{pmlr-v235-xie24e}. 
Recent work further shows that this separation alters the effective optimization geometry and admits an implicit constrained interpretation connected to normalized steepest descent and Frank--Wolfe-type methods \citep{pmlr-v235-xie24e}. 
Adan extends this line by incorporating a Nesterov-style momentum mechanism within an adaptive optimizer, while retaining the overall coordinate-wise adaptive structure \citep{xie2024adan}. Overall, this family serves as the primary coordinate-wise baseline against which LMO-based methods are compared.

\paragraph{From normalized gradients to sign-based updates.}
A second line of work replaces magnitude-sensitive steps with normalized or sign-based directions. Stochastic normalized gradient descent already emphasized that the gradient direction, rather than its magnitude, can be the key object of analysis \citep{hazan2015beyond}. signSGD and Signum push this idea further by using sign-based updates, retaining strong optimization performance while significantly simplifying the update rule \citep{pmlr-v80-bernstein18a}. 
Lion continues this direction by combining sign-based updates with momentum while discarding second-moment adaptation \citep{chen2023symbolic}. As a result, its behavior is more closely aligned with normalized or dual-norm steepest descent than with classical adaptive methods. This interpretation has been reinforced by recent analyses: some view Lion with decoupled weight decay as solving an $\ell_\infty$-constrained problem \citep{chen2024lion}, while others establish stochastic convergence rates, showing the standard $\mathcal{O}(\varepsilon^{-4})$ complexity under gradient- or $\ell_1$-type stationarity, with variance-reduced variants improving to $\mathcal{O}(\varepsilon^{-3})$ \citep{dong2024convergence,jiang2025convergence}. 
Overall, this line of work demonstrates that sign-based methods are not merely empirical heuristics, but instances of a broader norm-aware optimization principle.

\paragraph{Matrix-structured and geometry-aware optimization.}
A third line of work emphasizes that the relevant geometry in deep learning is often neither purely Euclidean nor purely coordinate-wise. Early structure-aware preconditioners such as Shampoo exploit tensor structure rather than flattening parameters into vectors \citep{gupta2018shampoo}. More recent work makes this perspective explicit at the level of optimization geometry. For example, \citet{bernstein2024old} reinterpret several methods as steepest descent under non-Euclidean norms, while \citet{bernstein2025modular} develop layerwise duality maps based on operator norms for general neural architectures.  
Muon fits naturally into this line, as its orthogonalized updates operate on matrix-shaped parameters and are closely tied to spectral, nuclear, or operator-norm geometry \citep{jordan2024muon}. Recent analyses establish its convergence behavior, identify regimes where it can outperform gradient descent, and interpret Muon with decoupled weight decay as solving a spectrally constrained problem \citep{shen2025convergenceanalysismuon,chang2025convergencemuon,chen2025muon}. AdaMuon further combines this geometry with element-wise adaptivity, illustrating that adaptive and matrix-aware designs can be unified \citep{si2025adamuonadaptivemuonoptimizer}. 
From a comparative perspective, this line is particularly relevant because the underlying norm geometry becomes part of the algorithm design, leading to stationarity measures that range from gradient-based criteria to constrained Frank--Wolfe gap interpretations.

\paragraph{Frank--Wolfe, conditional gradients, and LMO-based viewpoints.}
The Frank--Wolfe literature provides the natural constrained counterpart to normalized and LMO-based updates. Classical Frank--Wolfe theory relies on duality-gap certificates for projection-free convex optimization \citep{jaggi2013revisiting}, and its nonconvex extension shows that the Frank--Wolfe gap remains a meaningful stationarity measure beyond convexity \citep{lacoste2016convergence}. Building on this foundation, stochastic and variance-reduced conditional-gradient methods have progressively improved the oracle complexity of nonconvex projection-free optimization, including stochastic Frank--Wolfe, one-sample variants, SFW++, and SARAH-based methods \citep{reddi2016stochastic,zhang2020one,hassani2020stochasticconditionalgradient,beznosikov2024sarah}. 
In parallel, recent deep-learning work reinterprets modern optimizers through this lens. Norm-constrained LMO methods unify a range of update rules within a stochastic framework and demonstrate that LMO-type updates can be applied even in unconstrained settings \citep{pmlr-v267-pethick25a}. Practical instantiations often employ layer-wise norm constraints and structured updates, enabling flexible, layer-dependent behavior at the cost of additional tuning. Closely related work further shows that Lion and Muon with decoupled weight decay can be interpreted as instances of stochastic Frank--Wolfe over $\ell_\infty$- or operator-norm-type feasible sets \citep{sfyraki2026lionsmuonsoptimizationstochastic}. This line of work is central to our setting, as it provides the constrained-optimization language that connects normalized updates, modern LMO-based optimizers, and projection-free notions of stationarity within a unified perspective.

\paragraph{Variance reduction and accelerated stochastic first-order methods.}
From a rate perspective, variance reduction remains the canonical approach for improving stochastic nonconvex first-order complexity from the SGD-like $\mathcal{O}(\varepsilon^{-4})$ regime to the widely studied $\mathcal{O}(\varepsilon^{-3})$ regime. SPIDER is a central reference, and its normalized-gradient variant is particularly relevant as it already connects recursive estimators with normalized updates \citep{NEURIPS2018_1543843a}. STORM further shows that momentum-like recursions can achieve similar improvements without the large batch sizes often required by variance-reduced methods \citep{NEURIPS2019_b8002139}. This idea has recently reappeared in large-model optimization through MARS, which unifies preconditioned updates with stochastic recursive momentum and recovers AdamW-, Lion-, and Shampoo-style methods within a common framework \citep{pmlr-v267-yuan25f}. A similar pattern also appears in recent analyses of Lion and Muon, where STORM-type correction terms are incorporated into sign-based or orthogonalized momentum updates to improve the stochastic rate \citep{jiang2025convergence,chang2025convergencemuon}.  
The key takeaway is that improved iteration complexity is already achievable for modern optimizers, but typically comes at the cost of additional estimator state, gradient-difference corrections, or multiple gradient evaluations per iteration.

\paragraph{Extrapolation, lookahead, and transported gradients.}
A complementary route to acceleration is to change where or how first-order information is evaluated along the optimization trajectory. Classical Nesterov acceleration forms updates from an extrapolated sequence rather than from the current iterate alone, thereby introducing a trajectory-aware form of momentum \citep{nesterov2013introductory}. In modern deep learning, this intuition appears in several practical forms. One is approximate Nesterov momentum, where lookahead behavior is encoded algebraically without requiring a separate gradient evaluation at a distinct extrapolated point; adaptive examples in this direction include Adan \citep{xie2024adan}. Another is Lookahead, which maintains fast and slow weights and periodically updates the slow sequence toward the trajectory generated by an inner optimizer \citep{zhang2019lookahead}.  

The transported-gradient line makes this trajectory-aware viewpoint even more explicit. The original IGT work introduced the idea of correcting gradient staleness by transporting previously computed gradient information along the optimization trajectory without explicit Hessian computations \citep{NEURIPS2019_1dba5eed}. Building on this idea, \citet{cutkosky2020momentum} showed that, under bounded second-derivative assumptions, a transported-gradient modification of normalized SGD attains an improved $\mathcal{O}(\varepsilon^{-3.5})$ stochastic complexity while preserving one stochastic gradient evaluation per iteration. Conceptually, these methods are related through the use of extrapolation or synchronization along the optimization path, but they differ in how auxiliary points are used: Nesterov-style methods modify the momentum/query construction, Lookahead explicitly couples fast and slow iterate sequences, whereas IGT evaluates stochastic gradients at a transported point to reduce momentum lag. Existing analyses in this broader extrapolation family, however, remain largely restricted to Euclidean or normalized settings and do not directly address general LMO queries, operator-norm geometries, or constrained gap-based stationarity.

\paragraph{Positioning of our work.}
Against this background, our contribution goes beyond a marginal rate improvement within a fixed optimizer family. Algorithmically, we extend implicit gradient transport from Euclidean normalized updates to general LMO-based methods, thereby unifying normalized, sign-based, and matrix-structured optimizers under a common mechanism. Analytically, we compare plain stochastic LMO, variance-reduced LMO, and LMO-IGT under a single stationarity measure, rather than switching between gradient norms in unconstrained settings and Frank--Wolfe gaps in constrained ones. This distinction is crucial: prior work typically improves either the update geometry or the iteration complexity within a fixed measure, whereas our framework enables comparisons across both geometry and stationarity notions on a common scale. As a result, the $\mathcal{O}(\varepsilon^{-4})$, $\mathcal{O}(\varepsilon^{-3})$, and $\mathcal{O}(\varepsilon^{-3.5})$ guarantees can be placed on a single comparison axis.

\section{Useful Lemmas}
\label{sec:useful-lemmas}

When \(\lambda>0\), we write $\mathcal P:=\lambda^{-1}\mathcal C.$ Recall also that
\[
\Psi_{\mathcal C,\lambda}(w)
:=
\sup_{v\in\mathcal C}\langle -\nabla F(w),\,v-\lambda w\rangle,
\qquad
R:=\operatorname{diam}(\mathcal C).
\]

\begin{lemma}[Geometry of a single LMO step]
\label{lem:common-step-geometry}
Let
\[
w_{t+1}=(1-\lambda\eta_t)w_t+\eta_t v_t,
\qquad
v_t\in\mathcal C,
\]
where either \(\lambda=0\), or \(\lambda>0\) and \(w_t\in\mathcal P=\lambda^{-1}\mathcal C\). Then
\[
\|v_t-\lambda w_t\|\le R,
\qquad
\|w_{t+1}-w_t\|\le \eta_t R.
\]
\end{lemma}

\begin{proof}
If \(\lambda>0\), then \(\lambda w_t\in\mathcal C\), and since \(v_t\in\mathcal C\),
\[
\|v_t-\lambda w_t\|\le \operatorname{diam}(\mathcal C)=R.
\]
If \(\lambda=0\), then \(\|v_t-\lambda w_t\|=\|v_t\|\le R\) because \(0\in\mathcal C\).

Finally,
\[
w_{t+1}-w_t=\eta_t(v_t-\lambda w_t),
\]
hence
\[
\|w_{t+1}-w_t\|\le \eta_t R.
\qedhere
\]
\end{proof}

\begin{lemma}[Descent rule]
\label{lem:common-descent}
Suppose Assumption~\ref{asmp:l_smooth} holds. Let \(g_t\in\mathbb R^d\), let
\[
v_t\in\arg\min_{v\in\mathcal C}\langle g_t,v\rangle,
\qquad
w_{t+1}=(1-\lambda\eta_t)w_t+\eta_t v_t,
\]
and define
\[
\hat\epsilon_t:=g_t-\nabla F(w_t).
\]
Then
\begin{align}
F(w_{t+1})
&\le
F(w_t)
-\eta_t\Psi_{\mathcal C,\lambda}(w_t)
+\eta_t R\|\hat\epsilon_t\|
+\frac{L}{2}\eta_t^2R^2.
\label{eq:common-descent}
\end{align}
Consequently, if \(\eta_t=\eta\) for every \(t=0,1,\ldots,T-1\), then
\begin{align}
\frac{1}{T}\sum_{t=0}^{T-1}\mathbb E\big[\Psi_{\mathcal C,\lambda}(w_t)\big]
\le
\frac{\Delta_F}{T\eta}
+\frac{R}{T}\sum_{t=0}^{T-1}\mathbb E\|\hat\epsilon_t\|
+\frac{L}{2}R^2\eta,
\label{eq:common-avg-descent}
\end{align}
where \(\Delta_F:=F(w_0)-F^*\).
\end{lemma}

\begin{proof}
Let
\[
\hat v_t\in\arg\max_{v\in\mathcal C}\langle -\nabla F(w_t),v\rangle.
\]
By \(L\)-smoothness of \(F\),
\begin{align*}
F(w_{t+1})
&\le
F(w_t)
+\eta_t\langle \nabla F(w_t),\,v_t-\lambda w_t\rangle
+\frac{L}{2}\eta_t^2\|v_t-\lambda w_t\|^2 \\
&\le
F(w_t)
+\eta_t\langle g_t,\,v_t-\lambda w_t\rangle
+\eta_t\langle \nabla F(w_t)-g_t,\,v_t-\lambda w_t\rangle
+\frac{L}{2}\eta_t^2R^2 \\
&\le
F(w_t)
+\eta_t\langle g_t,\,\hat v_t-\lambda w_t\rangle
+\eta_t\langle \nabla F(w_t)-g_t,\,v_t-\lambda w_t\rangle
+\frac{L}{2}\eta_t^2R^2 \\
&=
F(w_t)
+\eta_t\langle \nabla F(w_t),\,\hat v_t-\lambda w_t\rangle
+\eta_t\langle \nabla F(w_t)-g_t,\,v_t-\hat v_t\rangle
+\frac{L}{2}\eta_t^2R^2 \\
&=
F(w_t)
-\eta_t\Psi_{\mathcal C,\lambda}(w_t)
+\eta_t\langle \nabla F(w_t)-g_t,\,v_t-\hat v_t\rangle
+\frac{L}{2}\eta_t^2R^2 \\
&\le
F(w_t)
-\eta_t\Psi_{\mathcal C,\lambda}(w_t)
+\eta_t R\|\hat\epsilon_t\|
+\frac{L}{2}\eta_t^2R^2,
\end{align*}
where we used Lemma~\ref{lem:common-step-geometry} for \(\|v_t-\lambda w_t\|\le R\), the LMO optimality
\(\langle g_t,v_t\rangle\le \langle g_t,\hat v_t\rangle\), and \(\|v_t-\hat v_t\|\le R\) since both points lie in \(\mathcal C\).

Summing over \(t=0,\dots,T-1\), taking expectations, and dividing by \(T\eta\) gives \eqref{eq:common-avg-descent}.
\end{proof}

\begin{lemma}[IGT extrapolation geometry]
\label{lem:common-igt-geometry}
Suppose
\[
x_{t+1}=(1-\lambda\eta_{1,t})w_t+\eta_{1,t}v_t,
\qquad
w_{t+1}=(1-\lambda\eta_{2,t})w_t+\eta_{2,t}v_t.
\]
Then
\[
x_{t+1}-w_{t+1}
=
(\eta_{1,t}-\eta_{2,t})(v_t-\lambda w_t).
\]
If \(\eta_{1,t}=\eta_1\) and \(\eta_{2,t}=\eta_2\) for every \(t=0,1,\ldots,T-1\), then for every \(t\ge 1\),
\[
w_t-w_{t-1}=\eta_2(v_{t-1}-\lambda w_{t-1}),
\qquad
x_t-w_t=(\eta_1-\eta_2)(v_{t-1}-\lambda w_{t-1}),
\]
and hence
\[
x_t
=
w_t+C(w_t-w_{t-1}),
\qquad
C:=\frac{\eta_1-\eta_2}{\eta_2}.
\]
Moreover,
\[
\|w_t-w_{t-1}\|\le \eta_2 R,
\qquad
\|x_t-w_t\|\le |\eta_1-\eta_2|R.
\]
\end{lemma}

\begin{proof}
The identities follow directly from the definitions of \(x_{t+1}\) and \(w_{t+1}\). The norm bounds follow from Lemma~\ref{lem:common-step-geometry}.
\end{proof}

\begin{lemma}[Second-order remainder]
\label{lem:common-taylor}
Assume Assumption~\ref{asmp:sec_smooth}. Define
\[
Z(x,y)
:=
\int_0^1
\bigl(\nabla^2F(y+\tau(x-y))-\nabla^2F(y)\bigr)(x-y)\,d\tau.
\]
Then
\[
\nabla F(x)=\nabla F(y)+\nabla^2F(y)(x-y)+Z(x,y),
\]
and
\[
\|Z(x,y)\|\le \rho\|x-y\|^2.
\]
\end{lemma}

\begin{proof}
The identity is the integral form of Taylor's theorem. By Assumption~\ref{asmp:sec_smooth},
\begin{align*}
\|Z(x,y)\|
&\le
\int_0^1
\|\nabla^2F(y+\tau(x-y))-\nabla^2F(y)\|_{\mathrm{op}}\,\|x-y\|\,d\tau \\
&\le
\int_0^1 \rho\tau \|x-y\|^2\,d\tau
\le
\rho\|x-y\|^2.
\qedhere
\end{align*}
\end{proof}

\begin{lemma}[Momentum-weighted martingale bound]
\label{lem:common-martingale}
Let \(\mathcal F_t:=\sigma(\xi_0,\dots,\xi_t)\) be the natural filtration. Suppose \(\{\epsilon_t\}_{t\ge 0}\) is adapted to \(\{\mathcal F_t\}_{t\ge 0}\) and satisfies
\[
\mathbb E[\epsilon_t\mid \mathcal F_{t-1}]=0,
\qquad
\mathbb E[\|\epsilon_t\|^2\mid \mathcal F_{t-1}]\le \sigma^2.
\]
Fix \(0\le \beta_1\le \beta_2<1\). For \(t\ge 1\), define
\[
N_t
:=
\beta_1\beta_2^{t-1}\epsilon_0
+\beta_1(1-\beta_2)\sum_{s=1}^{t-1}\beta_2^{t-1-s}\epsilon_s
+(1-\beta_1)\epsilon_t.
\]
Then
\begin{align}
\mathbb E\|N_t\|^2
&\le
\beta_1^2\beta_2^{2t-2}\sigma^2
+\beta_1^2(1-\beta_2)\sigma^2
+(1-\beta_1)^2\sigma^2,
\label{eq:common-martingale-second}
\\
\mathbb E\|N_t\|
&\le
\beta_1\beta_2^{t-1}\sigma
+\beta_1\sqrt{1-\beta_2}\,\sigma
+(1-\beta_1)\sigma.
\label{eq:common-martingale-first}
\end{align}
If additionally \(N_0:=\epsilon_0\), then
\begin{align}
\frac{1}{T}\sum_{t=0}^{T-1}\mathbb E\|N_t\|
&\le
\frac{\sigma}{T}
+\sigma\left(
(1-\beta_1)
+\beta_1\sqrt{1-\beta_2}
+\frac{\beta_1}{T(1-\beta_2)}
\right).
\label{eq:common-martingale-avg}
\end{align}
\end{lemma}

\begin{proof}
Because \(\{\epsilon_t\}\) is a martingale difference sequence, the cross terms vanish:
\[
\mathbb E\langle \epsilon_r,\epsilon_s\rangle=0,
\qquad r\neq s.
\]
Hence
\begin{align*}
\mathbb E\|N_t\|^2
&\le
\beta_1^2\beta_2^{2t-2}\sigma^2
+\beta_1^2(1-\beta_2)^2
\sum_{s=1}^{t-1}\beta_2^{2t-2-2s}\sigma^2
+(1-\beta_1)^2\sigma^2 \\
&\le
\beta_1^2\beta_2^{2t-2}\sigma^2
+\beta_1^2(1-\beta_2)\sigma^2
+(1-\beta_1)^2\sigma^2,
\end{align*}
where we used
\[
(1-\beta_2)^2\sum_{s=1}^{t-1}\beta_2^{2t-2-2s}
\le
(1-\beta_2)^2\sum_{k=0}^{\infty}\beta_2^{2k}
=
\frac{1-\beta_2}{1+\beta_2}
\le
1-\beta_2.
\]
This proves \eqref{eq:common-martingale-second}.

Applying Jensen's inequality and \(\sqrt{a+b+c}\le \sqrt a+\sqrt b+\sqrt c\) gives
\eqref{eq:common-martingale-first}. Finally, averaging over \(t=1,\dots,T-1\), using
\[
\sum_{t=1}^{T-1}\beta_2^{t-1}\le \frac{1}{1-\beta_2},
\]
and \(\mathbb E\|N_0\|\le \sigma\), yields \eqref{eq:common-martingale-avg}.
\end{proof}

\begin{lemma}[Centered gradient-difference bound]
\label{lem:common-centered}
Under Assumption~\ref{asmp:avg_sample_smth}, for all \(x,y\) in the region of interest,
\[
\mathbb E_{\xi}\!\left[
\left\|
\bigl(\nabla f(x;\xi)-\nabla f(y;\xi)\bigr)
-
\bigl(\nabla F(x)-\nabla F(y)\bigr)
\right\|^2
\right]
\le
L^2\|x-y\|^2.
\]
\end{lemma}

\begin{proof}
Let
\[
Z:=\nabla f(x;\xi)-\nabla f(y;\xi).
\]
Then \(\mathbb E[Z]=\nabla F(x)-\nabla F(y)\), and therefore
\[
\mathbb E\|Z-\mathbb EZ\|^2
=
\mathbb E\|Z\|^2-\|\mathbb EZ\|^2
\le
\mathbb E\|Z\|^2
\le
L^2\|x-y\|^2.
\qedhere
\]
\end{proof}

\paragraph{Approximate Nesterov momentum.}
\label{app:nesterov-reparam}

Before using the term Nesterov momentum, we briefly specify the convention adopted in this appendix. In the classical exact Nesterov scheme, the gradient is evaluated at a lookahead point rather than at the current iterate. A generic form is
\[
y_t=w_t+\gamma_t(w_t-w_{t-1}),
\qquad
m_t=\beta_t m_{t-1}+(1-\beta_t)\nabla f(y_t;\xi_t).
\]

In practice, many deep-learning implementations use an approximate Nesterov form that does not require a separate gradient evaluation at a distinct lookahead point. Let \(z_t\) denote the point at which the stochastic gradient is evaluated. We then define
\[
m_t=\beta_{2,t}m_{t-1}+(1-\beta_{2,t})\nabla f(z_t;\xi_t),
\]
and
\[
g_t=\bar\beta_{1,t}m_t+(1-\bar\beta_{1,t})\nabla f(z_t;\xi_t).
\]

Substituting the recursion for \(m_t\) into the definition of \(g_t\) gives
\[
g_t
=
\bar\beta_{1,t}\beta_{2,t}m_{t-1}
+
\bigl(1-\bar\beta_{1,t}\beta_{2,t}\bigr)\nabla f(z_t;\xi_t).
\]
Therefore the approximate Nesterov query is exactly the same as the two-momentum query
\[
g_t
=
\beta_{1,t}m_{t-1}
+
(1-\beta_{1,t})\nabla f(z_t;\xi_t)
\]
with
\[
\beta_{1,t}=\bar\beta_{1,t}\beta_{2,t}.
\]

In particular, if \(0\le \bar\beta_{1,t}\le 1\) and \(0\le \beta_{2,t}<1\), then
\[
0\le \beta_{1,t}\le \beta_{2,t}<1.
\]
When \(z_t=w_t\), this recovers the current-point stochastic-LMO specialization. When \(z_t=x_t\), it recovers the transported-point LMO-IGT specialization. For this reason, we do not state separate Nesterov-specific convergence theorems in the appendix.

\section{Theoretical Analysis of the LMO-IGT Class}
\label{app:igt}

Consider the Algorithm~\ref{alg:igt-lmo}. For every \(t=0,1,\ldots,T-1\), let
\[
\beta_{1,t}=\beta_1,
\qquad
\beta_{2,t}=\beta_2,
\qquad
\eta_{2,t}=\eta,
\qquad
\eta_{1,t}=\frac{\eta}{1-\beta_2}.
\]

Applying Lemma~\ref{lem:common-descent} with stepsize \(\eta\) gives
\begin{align}
\frac{1}{T}\sum_{t=0}^{T-1}\mathbb E\big[\Psi_{\mathcal C,\lambda}(w_t)\big]
\le
\frac{\Delta_F}{T\eta}
+\frac{R}{T}\sum_{t=0}^{T-1}\mathbb E\|\hat\epsilon_t\|
+\frac{L}{2}R^2\eta,
\label{eq:avg-descent}
\end{align}
where
\[
\hat\epsilon_t=g_t-\nabla F(w_t).
\]

\begin{lemma}
\label{lem:igt-hat}
Suppose Assumptions~\ref{asmp:l_smooth},~\ref{asmp:bound_var}, and~\ref{asmp:sec_smooth} hold. For every \(t=0,1,\ldots,T-1\), let
\[
\beta_{1,t}=\beta_1,
\qquad
\beta_{2,t}=\beta_2,
\qquad
\eta_{2,t}=\eta,
\qquad
\eta_{1,t}=\frac{\eta}{1-\beta_2}.
\]
Then
\begin{align}
\frac{1}{T}\sum_{t=0}^{T-1}\mathbb E\|\hat\epsilon_t\|
&\le
\sigma\left(
(1-\beta_1)
+\beta_1\sqrt{1-\beta_2}
+\frac{\beta_1}{T(1-\beta_2)}
+\frac{1}{T}
\right)
\nonumber\\
&\quad
+\frac{\beta_2-\beta_1}{1-\beta_2}LR\eta
+\left(
\frac{\beta_1}{1-\beta_2}
+\frac{\beta_2^2}{(1-\beta_2)^2}
\right)\rho R^2\eta^2.
\label{eq:hat-eps-bound}
\end{align}
\end{lemma}

\begin{proof}
Let
\[
\mathcal F_t=\sigma(\xi_0,\dots,\xi_t)
\]
be the natural filtration, and define
\[
\epsilon_t=\nabla f(x_t;\xi_t)-\nabla F(x_t).
\]
Since \(x_t\) is \(\mathcal F_{t-1}\)-measurable, Assumption~\ref{asmp:bound_var} implies
\[
\mathbb E[\epsilon_t\mid \mathcal F_{t-1}]=0,
\qquad
\mathbb E[\|\epsilon_t\|^2\mid \mathcal F_{t-1}]\le \sigma^2.
\]

We track
\[
\tilde\epsilon_t=m_t-\nabla F(w_t),
\qquad
\hat\epsilon_t=g_t-\nabla F(w_t).
\]
For \(t\ge 1\), define
\[
A_t=\nabla^2F(w_t)(w_t-w_{t-1}),
\qquad
Z_t^w=Z(w_{t-1},w_t),
\qquad
Z_t^x=Z(x_t,w_t),
\]
where \(Z(\cdot,\cdot)\) is from Lemma~\ref{lem:common-taylor}. By Lemma~\ref{lem:common-igt-geometry},
\[
x_t=w_t+C(w_t-w_{t-1}),
\qquad
C=\frac{\eta_1-\eta}{\eta}.
\]
Hence Lemma~\ref{lem:common-taylor} gives
\[
\nabla F(x_t)=\nabla F(w_t)+CA_t+Z_t^x,
\qquad
\nabla F(w_{t-1})=\nabla F(w_t)-A_t+Z_t^w.
\]

Substituting into
\[
m_t=\beta_2 m_{t-1}+(1-\beta_2)\bigl(\nabla F(x_t)+\epsilon_t\bigr)
\]
yields
\begin{align}
\tilde\epsilon_t
&=
\beta_2\tilde\epsilon_{t-1}
+\bigl(C(1-\beta_2)-\beta_2\bigr)A_t
+\beta_2 Z_t^w
+(1-\beta_2)(\epsilon_t+Z_t^x).
\label{eq:tilde-rec-app}
\end{align}
Similarly,
\[
g_t=\beta_1 m_{t-1}+(1-\beta_1)\bigl(\nabla F(x_t)+\epsilon_t\bigr)
\]
gives
\begin{align}
\hat\epsilon_t
&=
\beta_1\tilde\epsilon_{t-1}
+\bigl(C(1-\beta_1)-\beta_1\bigr)A_t
+\beta_1 Z_t^w
+(1-\beta_1)(\epsilon_t+Z_t^x).
\label{eq:hat-rec-app}
\end{align}

Now choose
\[
\eta_1=\frac{\eta}{1-\beta_2},
\qquad
C=\frac{\eta_1-\eta}{\eta}
=\frac{\beta_2}{1-\beta_2}.
\]
Then
\[
C(1-\beta_2)-\beta_2=0,
\]
so the first-order drift disappears from \eqref{eq:tilde-rec-app}:
\begin{align}
\tilde\epsilon_t
&=
\beta_2\tilde\epsilon_{t-1}
+\beta_2 Z_t^w
+(1-\beta_2)(\epsilon_t+Z_t^x).
\label{eq:tilde-rec-cancel-app}
\end{align}
Also
\[
C(1-\beta_1)-\beta_1
=
\frac{\beta_2-\beta_1}{1-\beta_2},
\]
hence
\begin{align}
\hat\epsilon_t
&=
\beta_1\tilde\epsilon_{t-1}
+\frac{\beta_2-\beta_1}{1-\beta_2}A_t
+\beta_1 Z_t^w
+(1-\beta_1)(\epsilon_t+Z_t^x).
\label{eq:hat-rec-cancel-app}
\end{align}

Because \(x_0=w_0\) and \(m_{-1}=\nabla f(w_0;\xi_0)\), we have
\[
\tilde\epsilon_0=\epsilon_0.
\]
Unrolling \eqref{eq:tilde-rec-cancel-app}, for \(t\ge 1\),
\[
\tilde\epsilon_{t-1}
=
\beta_2^{t-1}\epsilon_0
+
(1-\beta_2)\sum_{s=1}^{t-1}\beta_2^{t-1-s}\epsilon_s
+
\sum_{s=1}^{t-1}\beta_2^{t-s}Z_s^w
+
(1-\beta_2)\sum_{s=1}^{t-1}\beta_2^{t-1-s}Z_s^x.
\]
Substituting into \eqref{eq:hat-rec-cancel-app}, we decompose
\[
\hat\epsilon_t=N_t+B_t,
\]
where
\[
N_t
=
\beta_1\beta_2^{t-1}\epsilon_0
+
\beta_1(1-\beta_2)\sum_{s=1}^{t-1}\beta_2^{t-1-s}\epsilon_s
+
(1-\beta_1)\epsilon_t,
\]
and
\[
D_t
=
\frac{\beta_2-\beta_1}{1-\beta_2}A_t
+
\beta_1\sum_{s=1}^{t}\beta_2^{t-s}Z_s^w
+
\beta_1(1-\beta_2)\sum_{s=1}^{t-1}\beta_2^{t-1-s}Z_s^x
+
(1-\beta_1)Z_t^x.
\]

For \(t=0\), we set
\[
N_0=\hat\epsilon_0=\epsilon_0.
\]

By Assumption~\ref{asmp:l_smooth} and Lemma~\ref{lem:common-igt-geometry},
\[
\|A_t\|
\le
\|\nabla^2F(w_t)\|_{\mathrm{op}}\|w_t-w_{t-1}\|
\le
L\eta R.
\]
By Lemmas~\ref{lem:common-igt-geometry} and~\ref{lem:common-taylor},
\[
\|Z_t^w\|
\le
\rho\|w_t-w_{t-1}\|^2
\le
\rho\eta^2R^2,
\]
and
\[
\|Z_t^x\|
\le
\rho\|x_t-w_t\|^2
\le
\rho(\eta_1-\eta)^2R^2
=
\rho\left(\frac{\beta_2}{1-\beta_2}\right)^2\eta^2R^2.
\]

Using these bounds together with
\[
\sum_{s=1}^{t}\beta_2^{t-s}\le \frac{1}{1-\beta_2},
\qquad
\beta_1(1-\beta_2)\sum_{s=1}^{t-1}\beta_2^{t-1-s}+(1-\beta_1)\le 1,
\]
we obtain
\begin{align}
\|D_t\|
&\le
\frac{\beta_2-\beta_1}{1-\beta_2}L\eta R
+\frac{\beta_1}{1-\beta_2}\rho\eta^2R^2
+\rho(\eta_1-\eta)^2R^2
\notag\\
&=
\frac{\beta_2-\beta_1}{1-\beta_2}L\eta R
+\left(
\frac{\beta_1}{1-\beta_2}
+\frac{\beta_2^2}{(1-\beta_2)^2}
\right)\rho\eta^2R^2.
\label{eq:B-bound-app}
\end{align}

Lemma~\ref{lem:common-martingale} gives
\begin{align}
\frac{1}{T}\sum_{t=0}^{T-1}\mathbb E\|N_t\|
\le
\sigma\left(
(1-\beta_1)
+\beta_1\sqrt{1-\beta_2}
+\frac{\beta_1}{T(1-\beta_2)}
+\frac{1}{T}
\right).
\label{eq:igt-N-avg}
\end{align}

Finally,
\[
\frac{1}{T}\sum_{t=0}^{T-1}\mathbb E\|\hat\epsilon_t\|
\le
\frac{1}{T}\sum_{t=0}^{T-1}\mathbb E\|N_t\|
+\sup_{0\le t\le T-1}\|D_t\|.
\]
Combining \eqref{eq:B-bound-app} and \eqref{eq:igt-N-avg} proves \eqref{eq:hat-eps-bound}.
\end{proof}

Substituting Lemma~\ref{lem:igt-hat} into \eqref{eq:avg-descent} yields
\begin{align}
\frac{1}{T}\sum_{t=0}^{T-1}\mathbb E\big[\Psi_{\mathcal C,\lambda}(w_t)\big]
&\le
\frac{\Delta_F}{T\eta}
+R\sigma\left(
(1-\beta_1)
+\beta_1\sqrt{1-\beta_2}
+\frac{\beta_1}{T(1-\beta_2)}
+\frac{1}{T}
\right)
\nonumber\\
&\quad
+LR^2\eta\left(
\frac{\beta_2-\beta_1}{1-\beta_2}
+\frac12
\right)
+\rho R^3\eta^2\left(
\frac{\beta_1}{1-\beta_2}
+\frac{\beta_2^2}{(1-\beta_2)^2}
\right),
\label{eq:main-bound-app}
\end{align}
which proves Theorem~\ref{thm:igt-rsf}.

\subsection{Proof of Corollary~\ref{cor:igt-stochastic}}

\begin{proof}
Choose
\[
\eta=\frac{1}{RT^{5/7}},
\qquad
\beta_2=1-\frac{1}{T^{4/7}},
\qquad
1-\beta_1\in\left[\frac{1}{T^{4/7}},\frac{1}{T^{2/7}}\right].
\]
Then
\[
\frac{\Delta_F}{T\eta}
=
\frac{R\Delta_F}{T^{2/7}},
\qquad
1-\beta_1\le \frac{1}{T^{2/7}},
\qquad
\sqrt{1-\beta_2}=\frac{1}{T^{2/7}},
\qquad
\frac{1}{T(1-\beta_2)}=\frac{1}{T^{3/7}},
\qquad
\frac{1}{T}\le \frac{1}{T^{2/7}}.
\]
Also,
\[
\frac{\beta_2-\beta_1}{1-\beta_2}
=
\frac{(1-\beta_1)-(1-\beta_2)}{1-\beta_2}
\le
\frac{1/T^{2/7}}{1/T^{4/7}}
=
T^{2/7},
\]
so
\[
LR^2\eta\left(
\frac{\beta_2-\beta_1}{1-\beta_2}
+\frac12
\right)
=
O\!\left(\frac{RL}{T^{3/7}}\right).
\]
Finally,
\[
\frac{\beta_1}{1-\beta_2}\le T^{4/7},
\qquad
\frac{\beta_2^2}{(1-\beta_2)^2}\le T^{8/7},
\qquad
R^3\eta^2=\frac{R}{T^{10/7}},
\]
and hence
\[
\rho R^3\eta^2\left(
\frac{\beta_1}{1-\beta_2}
+\frac{\beta_2^2}{(1-\beta_2)^2}
\right)
\le
\rho R\left(\frac{1}{T^{6/7}}+\frac{1}{T^{2/7}}\right).
\]
Combining the above bounds yields
\[
\frac{1}{T}\sum_{t=0}^{T-1}\mathbb E\big[\Psi_{\mathcal C,\lambda}(w_t)\big]
\le
O\!\left(
\frac{R(\Delta_F+\sigma+\rho)}{T^{2/7}}
+\frac{RL}{T^{3/7}}
\right)
=
O(T^{-2/7}).
\]
\end{proof}

\subsection{Proof of Corollary~\ref{cor:igt-deterministic}}

\begin{proof}
When \(\sigma=0\), applying \eqref{eq:main-bound-app} with
\[
\eta=\frac{1}{RT^{1/2}},
\qquad
\beta_1=\beta_2=1-\frac{1}{T^{1/4}}
\]
gives
\begin{align*}
\frac{1}{T}\sum_{t=0}^{T-1}\mathbb E\big[\Psi_{\mathcal C,\lambda}(w_t)\big]
&\le
\frac{\Delta_F}{T\eta}
+\frac{L}{2}R^2\eta
+\rho R^3\eta^2\left(
\frac{\beta_1}{1-\beta_2}
+\frac{\beta_2^2}{(1-\beta_2)^2}
\right).
\end{align*}
Now
\[
\frac{\Delta_F}{T\eta}
=
\frac{R\Delta_F}{T^{1/2}},
\qquad
\frac{L}{2}R^2\eta
=
\frac{RL}{2T^{1/2}}.
\]
Also,
\[
\frac{\beta_1}{1-\beta_2}\le T^{1/4},
\qquad
\frac{\beta_2^2}{(1-\beta_2)^2}\le T^{1/2},
\qquad
R^3\eta^2=\frac{R}{T},
\]
hence
\[
\rho R^3\eta^2\left(
\frac{\beta_1}{1-\beta_2}
+\frac{\beta_2^2}{(1-\beta_2)^2}
\right)
\le
\rho R\left(\frac{1}{T^{3/4}}+\frac{1}{T^{1/2}}\right).
\]
Therefore
\[
\frac{1}{T}\sum_{t=0}^{T-1}\mathbb E\big[\Psi_{\mathcal C,\lambda}(w_t)\big]
\le
O\!\left(
\frac{R(\Delta_F+L+\rho)}{T^{1/2}}
\right)
=
O(T^{-1/2}),
\]
which proves Corollary~\ref{cor:igt-deterministic}.
\end{proof}

\section{Theoretical Analysis of the Stochastic LMO Class}
\label{app:vanilla}

Consider the following algorithm.

\begin{algorithm}[!htb]
\caption{Stochastic LMO Optimization}
\label{alg:vanilla-lmo_app}
\begin{algorithmic}[1]
\Require initial point \(w_0\), momentum buffer \(m_{-1}=\nabla f(w_0;\xi_0)\), parameters \(\{\beta_{1,t}\}_{t\ge 0}\), \(\{\beta_{2,t}\}_{t\ge 0}\), \(\lambda\ge 0\), and step sizes \(\{\eta_t\}_{t\ge 0}\)
\For{\(t=0,\ldots,T-1\)}
    \State \(g_t \gets \beta_{1,t}m_{t-1}+(1-\beta_{1,t})\nabla f(w_t;\xi_t)\)
    \State \(m_t \gets \beta_{2,t}m_{t-1}+(1-\beta_{2,t})\nabla f(w_t;\xi_t)\)
    \State \(v_t \gets \operatorname{LMO}_{\mathcal C}(g_t)\)
    \State \(w_{t+1} \gets (1-\lambda\eta_t)w_t+\eta_t v_t\)
\EndFor
\end{algorithmic}
\end{algorithm}

Algorithm~\ref{alg:vanilla-lmo_app} is the stochastic-LMO specialization of Algorithm~\ref{alg:unified-lmo-framework}. In particular, it corresponds to setting
\[
\alpha_{1,t}=0,
\qquad
\alpha_{2,t}=0,
\qquad
\eta_{1,t}=\eta_{2,t}=\eta_t
\]
for every \(t=0,1,\ldots,T-1\).

We analyze Algorithm~\ref{alg:vanilla-lmo_app} under the parameter choice
\[
\beta_{1,t}=\beta_1,
\qquad
\beta_{2,t}=\beta_2,
\qquad
\eta_t=\eta,
\qquad
t=0,1,\ldots,T-1.
\]

\begin{theorem}[Stochastic LMO method]
\label{thm:vanilla}
Suppose Assumptions~\ref{asmp:l_smooth}--\ref{asmp:bound_var} hold. Consider Algorithm~\ref{alg:vanilla-lmo_app} and assume that, for every \(t=0,1,\ldots,T-1\),
\[
\beta_{1,t}=\beta_1,
\qquad
\beta_{2,t}=\beta_2,
\qquad
\eta_t=\eta,
\]
where \(0\le \beta_1\le \beta_2<1\) and \(0\le \lambda\eta\le 1\). When \(\lambda>0\), assume \(w_0\in\mathcal P=\lambda^{-1}\mathcal C\). Let \(F^*\) be any finite lower bound on \(F\) over a region containing the iterates, and define
\[
\Delta_F=F(w_0)-F^*.
\]
Then
\begin{align}
\frac{1}{T}\sum_{t=0}^{T-1}
\mathbb E\big[\Psi_{\mathcal C,\lambda}(w_t)\big]
&\le
\frac{\Delta_F}{T\eta}
+R\sigma\left(
(1-\beta_1)
+\beta_1\sqrt{1-\beta_2}
+\frac{\beta_1}{T(1-\beta_2)}
+\frac{1}{T}
\right)
\nonumber\\
&\quad
+LR^2\eta\left(
\frac{\beta_1}{1-\beta_2}
+\frac12
\right).
\label{eq:vanilla-main}
\end{align}
\end{theorem}

\begin{corollary}[Equal momentum parameters]
\label{cor:vanilla-beta-equal}
Under the assumptions of Theorem~\ref{thm:vanilla}, suppose
\[
\beta_1=\beta_2=\beta.
\]
Then
\begin{align}
\frac{1}{T}\sum_{t=0}^{T-1}
\mathbb E\big[\Psi_{\mathcal C,\lambda}(w_t)\big]
&\le
\frac{\Delta_F}{T\eta}
+R\sigma\left(
(1+\beta)\sqrt{1-\beta}
+\frac{\beta}{T(1-\beta)}
+\frac{1}{T}
\right)
\nonumber\\
&\quad
+LR^2\eta\left(
\frac{\beta}{1-\beta}
+\frac12
\right),
\label{eq:vanilla-beta-equal}
\end{align}
where we used \(1-\beta\le \sqrt{1-\beta}\).
\end{corollary}

\begin{corollary}[Stochastic regime for stochastic LMO]
\label{cor:vanilla-stochastic}
Under the assumptions of Corollary~\ref{cor:vanilla-beta-equal}, suppose \(\sigma>0\). For sufficiently large \(T\), choose
\[
\eta=\frac{1}{RT^{3/4}},
\qquad
\beta=1-\frac{1}{T^{1/2}}.
\]
Then
\[
\frac{1}{T}\sum_{t=0}^{T-1}
\mathbb E\big[\Psi_{\mathcal C,\lambda}(w_t)\big]
\le
O\!\left(
\frac{R(\Delta_F+\sigma+L)}{T^{1/4}}
\right)
=
O(T^{-1/4}).
\]
\end{corollary}

\begin{corollary}[Deterministic regime for stochastic LMO]
\label{cor:vanilla-deterministic}
Under the assumptions of Corollary~\ref{cor:vanilla-beta-equal}, suppose \(\sigma=0\). Let \(\beta\in[0,1)\) be any fixed constant independent of \(T\), and choose
\[
\eta=\frac{1}{RT^{1/2}}.
\]
Then
\[
\frac{1}{T}\sum_{t=0}^{T-1}
\mathbb E\big[\Psi_{\mathcal C,\lambda}(w_t)\big]
\le
O\!\left(
\frac{R\bigl(\Delta_F+L/(1-\beta)\bigr)}{T^{1/2}}
\right)
=
O(T^{-1/2}).
\]
\end{corollary}

Applying Lemma~\ref{lem:common-descent} with \(\eta_t=\eta\) for every \(t=0,1,\ldots,T-1\) yields
\begin{align}
\frac{1}{T}\sum_{t=0}^{T-1}\mathbb E\big[\Psi_{\mathcal C,\lambda}(w_t)\big]
\le
\frac{\Delta_F}{T\eta}
+\frac{R}{T}\sum_{t=0}^{T-1}\mathbb E\|\hat\epsilon_t\|
+\frac{L}{2}R^2\eta,
\label{eq:vanilla-avg-descent}
\end{align}
where
\[
\hat\epsilon_t=g_t-\nabla F(w_t).
\]

\begin{lemma}
\label{lem:vanilla-hat}
Suppose Assumptions~\ref{asmp:l_smooth}--\ref{asmp:bound_var} hold. Assume that, for every \(t=0,1,\ldots,T-1\),
\[
\beta_{1,t}=\beta_1,
\qquad
\beta_{2,t}=\beta_2,
\qquad
\eta_t=\eta.
\]
Then
\begin{align}
\frac{1}{T}\sum_{t=0}^{T-1}\mathbb E\|\hat\epsilon_t\|
&\le
\frac{\sigma}{T}
+\sigma\left(
(1-\beta_1)
+\beta_1\sqrt{1-\beta_2}
+\frac{\beta_1}{T(1-\beta_2)}
\right)
+\frac{\beta_1}{1-\beta_2}LR\eta.
\label{eq:vanilla-hat-bound}
\end{align}
\end{lemma}

Substituting Lemma~\ref{lem:vanilla-hat} into \eqref{eq:vanilla-avg-descent} yields
\begin{align}
\frac{1}{T}\sum_{t=0}^{T-1}\mathbb E\big[\Psi_{\mathcal C,\lambda}(w_t)\big]
&\le
\frac{\Delta_F}{T\eta}
+R\sigma\left(
(1-\beta_1)
+\beta_1\sqrt{1-\beta_2}
+\frac{\beta_1}{T(1-\beta_2)}
+\frac1T
\right)
\nonumber\\
&\quad
+LR^2\eta\left(
\frac{\beta_1}{1-\beta_2}
+\frac12
\right).
\label{eq:vanilla-main-app}
\end{align}
Since \eqref{eq:vanilla-main-app} is exactly \eqref{eq:vanilla-main}, this proves Theorem~\ref{thm:vanilla}.

\subsection{Proof of Lemma~\ref{lem:vanilla-hat}}

\begin{proof}
Let
\[
\mathcal F_t=\sigma(\xi_0,\dots,\xi_t)
\]
be the natural filtration, and define
\[
\epsilon_t=\nabla f(w_t;\xi_t)-\nabla F(w_t).
\]
Since \(w_t\) is \(\mathcal F_{t-1}\)-measurable, Assumption~\ref{asmp:bound_var} implies
\[
\mathbb E[\epsilon_t\mid \mathcal F_{t-1}]=0,
\qquad
\mathbb E[\|\epsilon_t\|^2\mid \mathcal F_{t-1}]\le \sigma^2.
\]

Define
\[
\tilde\epsilon_t=m_t-\nabla F(w_t),
\qquad
\hat\epsilon_t=g_t-\nabla F(w_t).
\]
For \(t\ge 1\), let
\[
\delta_t=\nabla F(w_{t-1})-\nabla F(w_t).
\]
By Assumption~\ref{asmp:l_smooth} and Lemma~\ref{lem:common-step-geometry},
\[
\|\delta_t\|
\le
L\|w_t-w_{t-1}\|
\le
L\eta R.
\]

From
\[
m_t=\beta_2 m_{t-1}+(1-\beta_2)(\nabla F(w_t)+\epsilon_t),
\]
we obtain
\begin{align}
\tilde\epsilon_t
&=
\beta_2\tilde\epsilon_{t-1}
+\beta_2\delta_t
+(1-\beta_2)\epsilon_t.
\label{eq:vanilla-tilde-rec}
\end{align}
Similarly, from
\[
g_t=\beta_1 m_{t-1}+(1-\beta_1)(\nabla F(w_t)+\epsilon_t),
\]
we get
\begin{align}
\hat\epsilon_t
&=
\beta_1\tilde\epsilon_{t-1}
+\beta_1\delta_t
+(1-\beta_1)\epsilon_t.
\label{eq:vanilla-hat-rec}
\end{align}

Because \(m_{-1}=\nabla f(w_0;\xi_0)\), the first update gives
\[
m_0=\nabla f(w_0;\xi_0),
\qquad
\tilde\epsilon_0=\epsilon_0.
\]
Unrolling \eqref{eq:vanilla-tilde-rec}, for \(t\ge 1\),
\[
\tilde\epsilon_{t-1}
=
\beta_2^{t-1}\epsilon_0
+
\sum_{s=1}^{t-1}\beta_2^{t-s}\delta_s
+
(1-\beta_2)\sum_{s=1}^{t-1}\beta_2^{t-1-s}\epsilon_s.
\]
Substituting this into \eqref{eq:vanilla-hat-rec}, we decompose
\[
\hat\epsilon_t=N_t+B_t,
\]
where
\[
N_t
=
\beta_1\beta_2^{t-1}\epsilon_0
+
\beta_1(1-\beta_2)\sum_{s=1}^{t-1}\beta_2^{t-1-s}\epsilon_s
+
(1-\beta_1)\epsilon_t,
\]
and
\[
B_t
=
\beta_1\sum_{s=1}^{t}\beta_2^{t-s}\delta_s.
\]

For \(t=0\), we set
\[
N_0=\hat\epsilon_0=\epsilon_0,
\qquad
B_0=0.
\]

Using \(\|\delta_s\|\le L\eta R\),
\[
\|B_t\|
\le
\beta_1\sum_{s=1}^{t}\beta_2^{t-s}\|\delta_s\|
\le
\frac{\beta_1}{1-\beta_2}L\eta R.
\]
Lemma~\ref{lem:common-martingale} gives
\[
\frac{1}{T}\sum_{t=0}^{T-1}\mathbb E\|N_t\|
\le
\frac{\sigma}{T}
+
\sigma\left(
(1-\beta_1)
+\beta_1\sqrt{1-\beta_2}
+\frac{\beta_1}{T(1-\beta_2)}
\right).
\]
Therefore
\[
\frac{1}{T}\sum_{t=0}^{T-1}\mathbb E\|\hat\epsilon_t\|
\le
\frac{1}{T}\sum_{t=0}^{T-1}\mathbb E\|N_t\|
+\sup_{0\le t\le T-1}\|B_t\|,
\]
which proves \eqref{eq:vanilla-hat-bound}.
\end{proof}

\subsection{Proof of Corollary~\ref{cor:vanilla-stochastic}}

\begin{proof}
Applying Corollary~\ref{cor:vanilla-beta-equal} with
\[
\eta=\frac{1}{RT^{3/4}},
\qquad
\beta=1-\frac{1}{T^{1/2}},
\]
gives
\begin{align*}
\frac{1}{T}\sum_{t=0}^{T-1}\mathbb E\big[\Psi_{\mathcal C,\lambda}(w_t)\big]
&\le
\frac{\Delta_F}{T\eta}
+R\sigma\left(
(1+\beta)\sqrt{1-\beta}
+\frac{\beta}{T(1-\beta)}
+\frac{1}{T}
\right)
+LR^2\eta\left(
\frac{\beta}{1-\beta}
+\frac12
\right).
\end{align*}
Now
\[
\frac{\Delta_F}{T\eta}
=
\frac{R\Delta_F}{T^{1/4}},
\qquad
(1+\beta)\sqrt{1-\beta}
\le
\frac{2}{T^{1/4}},
\qquad
\frac{\beta}{T(1-\beta)}
\le
\frac{1}{T^{1/2}},
\qquad
\frac{1}{T}\le \frac{1}{T^{1/4}}.
\]
Also,
\[
LR^2\eta\left(
\frac{\beta}{1-\beta}
+\frac12
\right)
\le
LR^2\cdot \frac{1}{RT^{3/4}}
\left(
T^{1/2}+\frac12
\right)
=
O\!\left(\frac{RL}{T^{1/4}}\right).
\]
Combining the bounds yields
\[
\frac{1}{T}\sum_{t=0}^{T-1}\mathbb E\big[\Psi_{\mathcal C,\lambda}(w_t)\big]
\le
O\!\left(
\frac{R(\Delta_F+\sigma+L)}{T^{1/4}}
\right),
\]
which proves Corollary~\ref{cor:vanilla-stochastic}.
\end{proof}

\subsection{Proof of Corollary~\ref{cor:vanilla-deterministic}}

\begin{proof}
When \(\sigma=0\), Corollary~\ref{cor:vanilla-beta-equal} becomes
\[
\frac{1}{T}\sum_{t=0}^{T-1}
\mathbb E\big[\Psi_{\mathcal C,\lambda}(w_t)\big]
\le
\frac{\Delta_F}{T\eta}
+LR^2\eta\left(
\frac{\beta}{1-\beta}
+\frac12
\right).
\]
Choosing
\[
\eta=\frac{1}{RT^{1/2}},
\]
gives
\[
\frac{\Delta_F}{T\eta}
=
\frac{R\Delta_F}{T^{1/2}},
\qquad
LR^2\eta\left(
\frac{\beta}{1-\beta}
+\frac12
\right)
=
\frac{RL}{T^{1/2}}
\left(
\frac{\beta}{1-\beta}
+\frac12
\right).
\]
Since \(\beta\in[0,1)\) is fixed independently of \(T\), this yields
\[
\frac{1}{T}\sum_{t=0}^{T-1}
\mathbb E\big[\Psi_{\mathcal C,\lambda}(w_t)\big]
\le
O\!\left(
\frac{R\bigl(\Delta_F+L/(1-\beta)\bigr)}{T^{1/2}}
\right),
\]
which proves Corollary~\ref{cor:vanilla-deterministic}.
\end{proof}

\section{Theoretical Analysis of the Variance-Reduced Stochastic LMO Class}
\label{app:gvr}

We analyze the LMO-VR specialization of Algorithm~\ref{alg:unified-lmo-framework}, obtained by setting
\[
\eta_{1,t}=\eta_{2,t}=\eta_t
\]
for every \(t=0,1,\ldots,T-1\). Since \(x_0=w_0\), this implies \(x_t=w_t\) for every \(t=0,1,\ldots,T\). For notational convenience, we write \(w_{-1}=w_0\).

For every \(t=0,1,\ldots,T-1\), let
\[
\alpha_{1,t}=\alpha_1,
\qquad
\alpha_{2,t}=\alpha_2,
\qquad
\beta_{1,t}=\beta_1,
\qquad
\beta_{2,t}=\beta_2,
\qquad
\eta_t=\eta.
\]

In this specialization, the query and momentum recursions are
\[
g_t=
\beta_1m_{t-1}
+(1-\beta_1)\nabla f(w_t;\xi_t)
+\alpha_1\bigl(\nabla f(w_t;\xi_t)-\nabla f(w_{t-1};\xi_t)\bigr),
\]
and
\[
m_t=
\beta_2m_{t-1}
+(1-\beta_2)\nabla f(w_t;\xi_t)
+\alpha_2\bigl(\nabla f(w_t;\xi_t)-\nabla f(w_{t-1};\xi_t)\bigr).
\]

Applying Lemma~\ref{lem:common-descent} gives
\begin{align}
\frac{1}{T}\sum_{t=0}^{T-1}\mathbb E\big[\Psi_{\mathcal C,\lambda}(w_t)\big]
\le
\frac{\Delta_F}{T\eta}
+\frac{R}{T}\sum_{t=0}^{T-1}\mathbb E\|\hat\epsilon_t\|
+\frac{L}{2}R^2\eta,
\label{eq:gvr-avg-descent}
\end{align}
where
\[
\hat\epsilon_t=g_t-\nabla F(w_t).
\]

\begin{lemma}[Tracking-error bound]
\label{lem:gvr-hat}
Suppose Assumptions~\ref{asmp:l_smooth}--\ref{asmp:avg_sample_smth} hold. Then
\begin{align}
\frac{1}{T}\sum_{t=0}^{T-1}\mathbb E\|\hat\epsilon_t\|
&\le
\frac{\sigma}{T}
+\sigma\left(
(1-\beta_1)
+\beta_1\sqrt{1-\beta_2}
+\frac{\beta_1}{T(1-\beta_2)}
\right)
\nonumber\\
&\quad
+LR\eta\left(
|\beta_1-\alpha_1|
+\frac{\beta_1|\beta_2-\alpha_2|}{1-\beta_2}
+|\alpha_1|
+\frac{\beta_1|\alpha_2|}{\sqrt{1-\beta_2}}
\right).
\label{eq:gvr-hat}
\end{align}
\end{lemma}

\begin{proof}
Let
\[
\mathcal F_t=\sigma(\xi_0,\dots,\xi_t)
\]
be the natural filtration, and define
\[
h_t=\nabla f(w_t;\xi_t),
\qquad
\epsilon_t=h_t-\nabla F(w_t).
\]
For \(t\ge 1\), define
\[
h_t^-=\nabla f(w_{t-1};\xi_t),
\qquad
\delta_t=\nabla F(w_{t-1})-\nabla F(w_t),
\]
and
\[
\zeta_t
=
\bigl(h_t-h_t^-\bigr)
-
\bigl(\nabla F(w_t)-\nabla F(w_{t-1})\bigr).
\]

By Assumption~\ref{asmp:bound_var},
\[
\mathbb E[\epsilon_t\mid \mathcal F_{t-1}]=0,
\qquad
\mathbb E[\|\epsilon_t\|^2\mid \mathcal F_{t-1}]\le \sigma^2.
\]
By Lemma~\ref{lem:common-centered},
\[
\mathbb E[\zeta_t\mid \mathcal F_{t-1}]=0,
\qquad
\mathbb E[\|\zeta_t\|^2\mid \mathcal F_{t-1}]
\le
L^2\|w_t-w_{t-1}\|^2.
\]
Moreover, by Assumption~\ref{asmp:l_smooth} and Lemma~\ref{lem:common-step-geometry},
\[
\|\delta_t\|
\le
L\|w_t-w_{t-1}\|
\le
L\eta R,
\]
and
\[
\mathbb E[\|\zeta_t\|^2\mid \mathcal F_{t-1}]
\le
L^2\eta^2R^2.
\]

Since \(w_{-1}=w_0\), \(m_{-1}=\nabla f(w_0;\xi_0)\), and
\[
\nabla f(w_0;\xi_0)-\nabla f(w_{-1};\xi_0)=0,
\]
the first step satisfies
\[
m_0=\nabla f(w_0;\xi_0),
\qquad
g_0=\nabla f(w_0;\xi_0).
\]
Therefore
\[
\tilde\epsilon_0=\hat\epsilon_0=\epsilon_0,
\qquad
\tilde\epsilon_t=m_t-\nabla F(w_t).
\]

For \(t\ge 1\), using
\[
h_t-h_t^-=\zeta_t-\delta_t,
\]
the recursion for \(m_t\) gives
\begin{align}
\tilde\epsilon_t
&=
\beta_2\tilde\epsilon_{t-1}
+
(\beta_2-\alpha_2)\delta_t
+
(1-\beta_2)\epsilon_t
+
\alpha_2\zeta_t.
\label{eq:gvr-buffer-rec}
\end{align}
Similarly, the recursion for \(g_t\) gives
\begin{align}
\hat\epsilon_t
&=
\beta_1\tilde\epsilon_{t-1}
+
(\beta_1-\alpha_1)\delta_t
+
(1-\beta_1)\epsilon_t
+
\alpha_1\zeta_t.
\label{eq:gvr-query-rec}
\end{align}

Unrolling \eqref{eq:gvr-buffer-rec} and substituting into \eqref{eq:gvr-query-rec}, we decompose
\[
\hat\epsilon_t=N_t+Z_t+B_t,
\]
where
\begin{align*}
N_t
&=
\beta_1\beta_2^{t-1}\epsilon_0
+
\beta_1(1-\beta_2)\sum_{s=1}^{t-1}\beta_2^{t-1-s}\epsilon_s
+
(1-\beta_1)\epsilon_t,\\
Z_t
&=
\beta_1\alpha_2\sum_{s=1}^{t-1}\beta_2^{t-1-s}\zeta_s
+
\alpha_1\zeta_t,\\
B_t
&=
\beta_1(\beta_2-\alpha_2)\sum_{s=1}^{t-1}\beta_2^{t-1-s}\delta_s
+
(\beta_1-\alpha_1)\delta_t.
\end{align*}

For \(t=0\), we set
\[
N_0=\hat\epsilon_0=\epsilon_0,
\qquad
Z_0=0,
\qquad
B_0=0.
\]

Lemma~\ref{lem:common-martingale} gives
\begin{align}
\frac{1}{T}\sum_{t=0}^{T-1}\mathbb E\|N_t\|
&\le
\frac{\sigma}{T}
+\sigma\left(
(1-\beta_1)
+\beta_1\sqrt{1-\beta_2}
+\frac{\beta_1}{T(1-\beta_2)}
\right).
\label{eq:gvr-N}
\end{align}

Using \(\|\delta_s\|\le L\eta R\),
\begin{align}
\|B_t\|
&\le
\beta_1|\beta_2-\alpha_2|
\sum_{s=1}^{t-1}\beta_2^{t-1-s}\|\delta_s\|
+
|\beta_1-\alpha_1|\,\|\delta_t\|
\notag\\
&\le
LR\eta\left(
\frac{\beta_1|\beta_2-\alpha_2|}{1-\beta_2}
+
|\beta_1-\alpha_1|
\right).
\label{eq:gvr-B}
\end{align}

Since \(\{\zeta_t\}\) is a martingale-difference sequence and
\[
\mathbb E[\|\zeta_t\|^2\mid \mathcal F_{t-1}]
\le
L^2\eta^2R^2,
\]
the cross terms vanish and
\begin{align*}
\mathbb E\|Z_t\|^2
&\le
\beta_1^2\alpha_2^2
\sum_{s=1}^{t-1}\beta_2^{2t-2-2s}L^2\eta^2R^2
+
\alpha_1^2L^2\eta^2R^2 \\
&\le
\left(
\frac{\beta_1^2\alpha_2^2}{1-\beta_2}
+
\alpha_1^2
\right)L^2\eta^2R^2.
\end{align*}
Hence, by Jensen's inequality,
\begin{align}
\mathbb E\|Z_t\|
&\le
LR\eta\left(
\frac{\beta_1|\alpha_2|}{\sqrt{1-\beta_2}}
+
|\alpha_1|
\right).
\label{eq:gvr-Z}
\end{align}

Combining \eqref{eq:gvr-N}, \eqref{eq:gvr-B}, and \eqref{eq:gvr-Z}, and using
\[
\|\hat\epsilon_t\|\le \|N_t\|+\|Z_t\|+\|B_t\|,
\]
proves \eqref{eq:gvr-hat}.
\end{proof}

Substituting Lemma~\ref{lem:gvr-hat} into \eqref{eq:gvr-avg-descent} yields
\begin{align}
\frac{1}{T}\sum_{t=0}^{T-1}\mathbb E\big[\Psi_{\mathcal C,\lambda}(w_t)\big]
&\le
\frac{\Delta_F}{T\eta}
+R\sigma\left(
(1-\beta_1)
+\beta_1\sqrt{1-\beta_2}
+\frac{\beta_1}{T(1-\beta_2)}
+\frac{1}{T}
\right)
\nonumber\\
&\quad
+LR^2\eta\left(
|\beta_1-\alpha_1|
+\frac{\beta_1|\beta_2-\alpha_2|}{1-\beta_2}
+|\alpha_1|
+\frac{\beta_1|\alpha_2|}{\sqrt{1-\beta_2}}
+\frac12
\right),
\end{align}
which proves Theorem~\ref{thm:gvr-main}.
\subsection{Proof of Corollary~\ref{cor:gvr-stochastic}}

\begin{proof}
Under the assumptions of Theorem~\ref{thm:gvr-main}, suppose
\[
0\le \alpha_1\le \beta_1,
\qquad
\alpha_2=\beta_2.
\]
Then the term involving \(|\beta_2-\alpha_2|\) vanishes, and
\[
|\beta_1-\alpha_1|+|\alpha_1|
=
(\beta_1-\alpha_1)+\alpha_1
=
\beta_1.
\]
Therefore Theorem~\ref{thm:gvr-main} reduces to
\begin{align}
\frac{1}{T}\sum_{t=0}^{T-1}\mathbb E\big[\Psi_{\mathcal C,\lambda}(w_t)\big]
&\le
\frac{\Delta_F}{T\eta}
+R\sigma\left(
(1-\beta_1)
+\beta_1\sqrt{1-\beta_2}
+\frac{\beta_1}{T(1-\beta_2)}
+\frac{1}{T}
\right)
\nonumber\\
&\quad
+LR^2\eta\left(
\beta_1
+\frac{\beta_1\beta_2}{\sqrt{1-\beta_2}}
+\frac12
\right).
\label{eq:gvr-stochastic-reduced}
\end{align}

Now choose
\[
\eta=\frac{1}{RT^{2/3}},
\qquad
\beta_2=1-\frac{1}{T^{2/3}},
\qquad
1-\beta_1\in\left[\frac{1}{T^{2/3}},\frac{1}{T^{1/3}}\right].
\]
Then
\[
\frac{\Delta_F}{T\eta}
=
\frac{R\Delta_F}{T^{1/3}},
\quad
1-\beta_1=O(T^{-1/3}),
\quad
\sqrt{1-\beta_2}=\frac{1}{T^{1/3}},
\quad
\frac{1}{T(1-\beta_2)}=\frac{1}{T^{1/3}},
\quad
\frac{1}{T}\le \frac{1}{T^{1/3}}.
\]
Moreover,
\[
LR^2\eta\left(
\beta_1
+\frac{\beta_1\beta_2}{\sqrt{1-\beta_2}}
+\frac12
\right)
\le
LR^2\cdot \frac{1}{RT^{2/3}}
\left(
1+T^{1/3}+\frac12
\right)
=
O\!\left(\frac{RL}{T^{1/3}}\right).
\]
Substituting these bounds into \eqref{eq:gvr-stochastic-reduced} yields
\[
\frac{1}{T}\sum_{t=0}^{T-1}\mathbb E\big[\Psi_{\mathcal C,\lambda}(w_t)\big]
\le
O\!\left(\frac{R(\Delta_F+\sigma+L)}{T^{1/3}}\right)
=
O\!\left(T^{-1/3}\right),
\]
which proves Corollary~\ref{cor:gvr-stochastic}.
\end{proof}

\section{Detailed Experimental Settings}
\label{app:experimental-details}

{All experiments were conducted in a Python 3.11.9 environment on a machine with a 64-core Intel Xeon Gold 6226R CPU at 2.90GHz, 512GB of memory, and an RTX 3090 GPU. All algorithms were implemented in PyTorch 2.4.1. Hyperparameters are tuned separately for each optimizer. We do not use learning-rate decay to isolate the effect of the optimizer update. }

\paragraph{Implementation of Muon.}

The momentum update rule of Muon proposed by \citet{jordan2024muon} is given by
\[
m_t = \beta_2 m_{t-1} + \nabla f(w_t; \xi_t).
\]
Compared to Algorithm~\ref{alg:unified-lmo-framework}, this is equivalent to scaling the loss function by a factor of $1/(1 - \beta_2)$. For a fair comparison with baselines, we adopt the momentum formulation used in Algorithm~\ref{alg:unified-lmo-framework}. Empirically, we observe no significant difference in performance between the two implementations.

\subsection{Multiclass Image Classification}
{For the image classification experiments in Section \ref{sec:experiments} and Appendix \ref{subsec:addCIFAR},} we follow the experimental setup of \citet{sfyraki2026lionsmuonsoptimizationstochastic} and train ResNet-18 on CIFAR-10 for 200 epochs. A single training run takes 1 to 2 hours on one RTX 3090 GPU. 

We first search over the learning-rate and weight-decay grids in Table~\ref{tab:img-cls-search}.

\begin{table}[!htb]
    \centering
    \small
    \caption{Hyperparameter search space for multiclass image classification.}
    \label{tab:img-cls-search}
    \begin{tabular}{ccc}
    \toprule
       & learning rate & weight decay  \\
        \midrule
    Muon  & \multirow{2}{*}{1e-1, 5e-2, 1e-2, 5e-3, 1e-3, 5e-4, 1e-4} & \multirow{10}{*}{0, 1e-4, 5e-4, 1e-3, 5e-3, 1e-2, 5e-2, 1e-1, 5e-1, 1.0}\\
    Muon-VR & &\\    \cmidrule{1-2}
    Muon-IGT    &  \multirow{5}{*}{1e-3, 5e-4, 1e-4, 5e-5, 1e-5, 5e-6, 1e-6}\\
    NIGT & &\\
    Lion & &\\
    Lion-VR & &\\
    AdamW & &\\ \cmidrule{1-2}
    Lion-IGT &1e-5, 5e-6, 1e-6, 5e-7, 1e-7, 5e-8 &\\
    \bottomrule
    \end{tabular}
    \label{tab:placeholder}
\end{table}

For the VR methods, we fix $\alpha_2$ equal to $\beta_2$ and search $\alpha_1$ over \{0, 0.1, 0.5, 0.9\}, subject to $\alpha_1$ being no larger than $\beta_1$.
The final hyperparameter choices used in the experiments are summarized in Table~\ref{tab:img-cls-settings}.

\begin{table*}[!htb]
    \centering
    \small
    \caption{Final hyperparameter settings for multiclass image classification.}
    \label{tab:img-cls-settings}
    \begin{tabular}{lcccccccc}
        \toprule
        & AdamW & NIGT & Lion & Lion-VR & Lion-IGT & Muon & Muon-VR & Muon-IGT \\
        \midrule
        Learning rate & 5e-4 & 5e-4 & 1e-4 & 1e-5 & 1e-5 & 5e-2 & 5e-2 & 5e-4 \\
        Weight decay & 5e-3 & 1e-1 & 1e-4 & 5e-2 & 1.0 & 5e-4 & 1e-3 & 5e-1 \\
        $\beta_1$ & 0.9 & 0.99 & 0.9 & 0.9 & 0.9 & 0.99 & 0.99 & 0.9 \\
        $\beta_2$ & 0.99 & 0.99 & 0.99 & 0.99 & 0.99 & 0.99 & 0.99 & 0.99 \\
        \midrule
        Model & \multicolumn{8}{c}{ResNet-18} \\
        Dataset & \multicolumn{8}{c}{CIFAR-10} \\
        Batch size & \multicolumn{8}{c}{128} \\
        Dropout & \multicolumn{8}{c}{0} \\
        \bottomrule
    \end{tabular}
\end{table*}

\subsection{Language Modeling}

{For the language modeling experiments in Appendix \ref{app:lang_model},} we follow the setup of \citet{sfyraki2026lionsmuonsoptimizationstochastic} and train nanoGPT~\citep{Karpathy2022} (MIT License) on the Shakespeare dataset~\citep{Karpathy2015} from scratch, for 5000 optimization steps. A single Shakespeare/nanoGPT run takes approximately 1 hour on one RTX 3090 GPU.

 We search over the hyperparameter grid in Table~\ref{tab:lm-search} and select the final hyperparameter settings in Table~\ref{tab:lm-settings}.

\begin{table}[!htb]
    \centering
    \small
    \caption{Hyperparameter search space for language modeling on Shakespeare.}
    \label{tab:lm-search}
    \begin{tabular}{ccc}
    \toprule
       & learning rate & weight decay  \\
        \midrule
    Muon  & \multirow{1}{*}{1e-1, 5e-2, 1e-2, 5e-3, 1e-3, 5e-4, 1e-4} & \multirow{6.2}{*}{0, 1e-4, 5e-4, 1e-3, 5e-3, 1e-2, 5e-2, 1e-1, 5e-1, 1.0}\\   \cmidrule{1-2}
    Muon-IGT    &  \multirow{3}{*}{1e-3, 5e-4, 1e-4, 5e-5, 1e-5, 5e-6, 1e-6}\\
    Lion & &\\
    AdamW & &\\ \cmidrule{1-2}
    Lion-IGT &1e-5, 5e-6, 1e-6, 5e-7, 1e-7, 5e-8 &\\
    \bottomrule
    \end{tabular}
\end{table}

\begin{table*}[!htb]
    \centering
    \small
    \caption{Final hyperparameter settings for language modeling on Shakespeare.}
    \label{tab:lm-settings}
    \begin{tabular}{lcccc}
        \toprule
        & AdamW & Lion & Muon & Muon-IGT \\
        \midrule
        Learning rate & 5e-4 & 1e-5 & 5e-2 & 5e-4 \\
        Weight decay & 5e-3 & 5e-2 & 5e-4 & 5e-1 \\
        Nesterov & - & - & Yes & - \\
       $\beta_1$ & 0.9 & 0.9 & 0.98$^2$ (Nesterov) & 0.9 \\
        \midrule
       $\beta_2$ & \multicolumn{4}{c}{0.98} \\
        Model & \multicolumn{4}{c}{nanoGPT (10M)} \\
        Dataset & \multicolumn{4}{c}{Shakespeare} \\
        Batch size & \multicolumn{4}{c}{64} \\
        Gradient accumulation steps & \multicolumn{4}{c}{1} \\
        Block size & \multicolumn{4}{c}{256} \\
        Dropout & \multicolumn{4}{c}{0.2} \\
        \bottomrule
    \end{tabular}
\end{table*}

\subsection{Large-scale Language Modeling}

{For the large-scale language modeling experiments in Appendix \ref{subsec:larger},} we adopt the model architecture of \citet{pmlr-v267-pethick25a} and train a 124M-parameter nanoGPT model on OpenWebText dataset~\citep{Gokaslan2019OpenWeb} from scratch. Concretely, we follow the modded-nanoGPT implementation, which incorporates several modern architectural choices: rotary positional embeddings instead of learned positional embeddings, RMSNorm~\citep{NEURIPS2019_1e8a1942} in place of LayerNorm, and the scaled ReLU  activation function instead of GELU. A single OpenWebText/124M nanoGPT run takes approximately 8 hours on one RTX 3090 GPU. Since the model does not exhibit clear signs of overfitting in this setting, we omit dropout and weight decay. %

{We search over the learning-rate grid in Table~\ref{tab:llm-search} and select the final hyperparameter settings reported in Table~\ref{tab:llm-settings}.}

\begin{table}[!htb]
    \centering
    \small
    \caption{Learning-rate search space for large-scale language modeling.}
    \label{tab:llm-search}
    \begin{tabular}{cc}
    \toprule
       & learning rate \\
        \midrule
    Muon  & \multirow{1}{*}{1e-1, 5e-2, 1e-2, 5e-3, 1e-3}\\   
    \cmidrule{1-2}
    Muon-IGT    &  \multirow{3}{*}{1e-3, 5e-4, 1e-4, 5e-5, 1e-5, 5e-6, 1e-6}\\
    Lion & \\
    AdamW & \\ 
    \bottomrule
    \end{tabular}
\end{table}

\begin{table}[!htb]
    \centering
    \small
    \caption{Final hyperparameter settings for large-scale language modeling on OpenWebText.}
    \label{tab:llm-settings}
    \begin{tabular}{lcccc}
        \toprule
        & AdamW & Lion & Muon & Muon-IGT \\
        \midrule
        Learning rate & 5e-4 & 1e-5 & 5e-2 & 5e-4 \\
        Nesterov & - & - & Yes & - \\
        $\beta_1$ & 0.95 & 0.95 & 0.99$^2$ (Nesterov) & 0.5 \\
        $\beta_2$ & 0.99 & 0.98 & 0.99 & 0.5 \\
        \midrule
        Weight decay & \multicolumn{4}{c}{0} \\
        Dropout & \multicolumn{4}{c}{0} \\
        Model & \multicolumn{4}{c}{nanoGPT (124M)} \\
        Dataset & \multicolumn{4}{c}{OpenWebText} \\
        Vocabulary size & \multicolumn{4}{c}{50304} \\
        Batch size & \multicolumn{4}{c}{16} \\
        Gradient accumulation steps & \multicolumn{4}{c}{16} \\
        Block size & \multicolumn{4}{c}{1024} \\
        \bottomrule
    \end{tabular}
\end{table}

\section{Additional Experimental Results}
\label{app:additional-experiments}
{We report additional results on multiclass image classification, language modeling, and large-scale language modeling. Hyperparameter choices and implementation details are provided in Appendix~\ref{app:experimental-details}.}

\subsection{Additional Results on Multiclass Image Classification}
\label{subsec:addCIFAR}
{For the CIFAR-10  task considered in the main text, we conduct additional experiments. For each optimizer, hyperparameters are selected based on the final test accuracy.}

\subsubsection{Ablation Study}
To understand the gains of Muon-IGT, we compare the three methods in Table~\ref{tab:ablation-components}, which differ in the use of double momentum and IGT. Figure~\ref{fig:appendix-ablation} illustrates the contribution of each component. Double momentum improves test accuracy over Muon, while further incorporating IGT yields additional gains in both test accuracy and test loss, resulting in the best overall performance.

\begin{table}[!htb]
    \centering
    \small
    \caption{Ablation variants for Muon-IGT.}
    \label{tab:ablation-components}
    \begin{tabular}{ccc}
        \toprule
        Method & Double momentum & IGT \\
        \midrule
        Muon & No & No \\
        Muon* & Yes & No \\
        Muon-IGT & Yes & Yes \\
        \bottomrule
    \end{tabular}
\end{table}

\begin{figure}[!htb]
    \centering
    \begin{subfigure}[b]{0.49\textwidth}
        \centering
        \includegraphics[width=\linewidth]{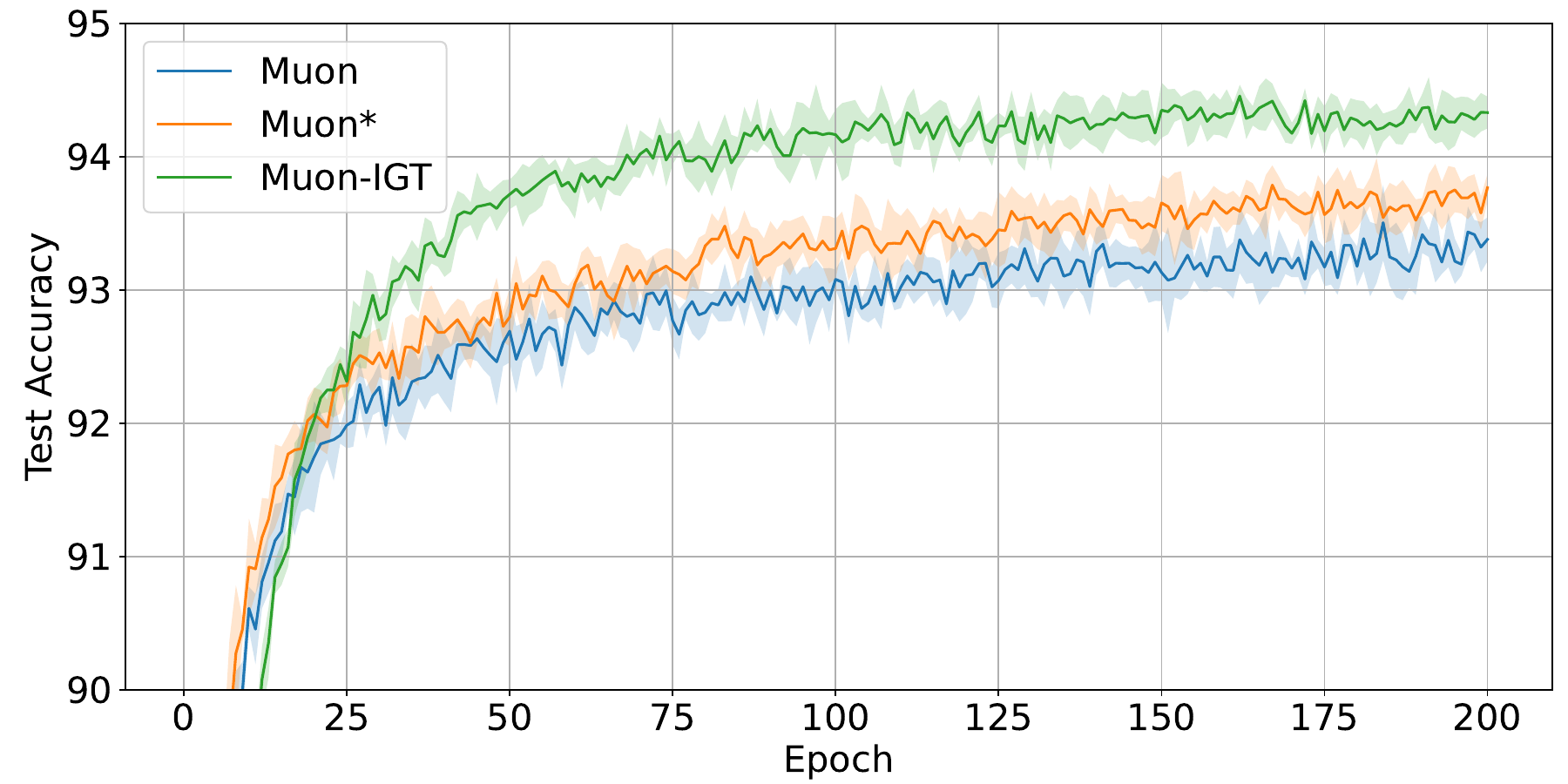}
        \caption{Test accuracy over epochs.}
        \label{fig:appendix-ablation-acc}
    \end{subfigure}
    \hfill
    \begin{subfigure}[b]{0.49\textwidth}
        \centering
        \includegraphics[width=\linewidth]{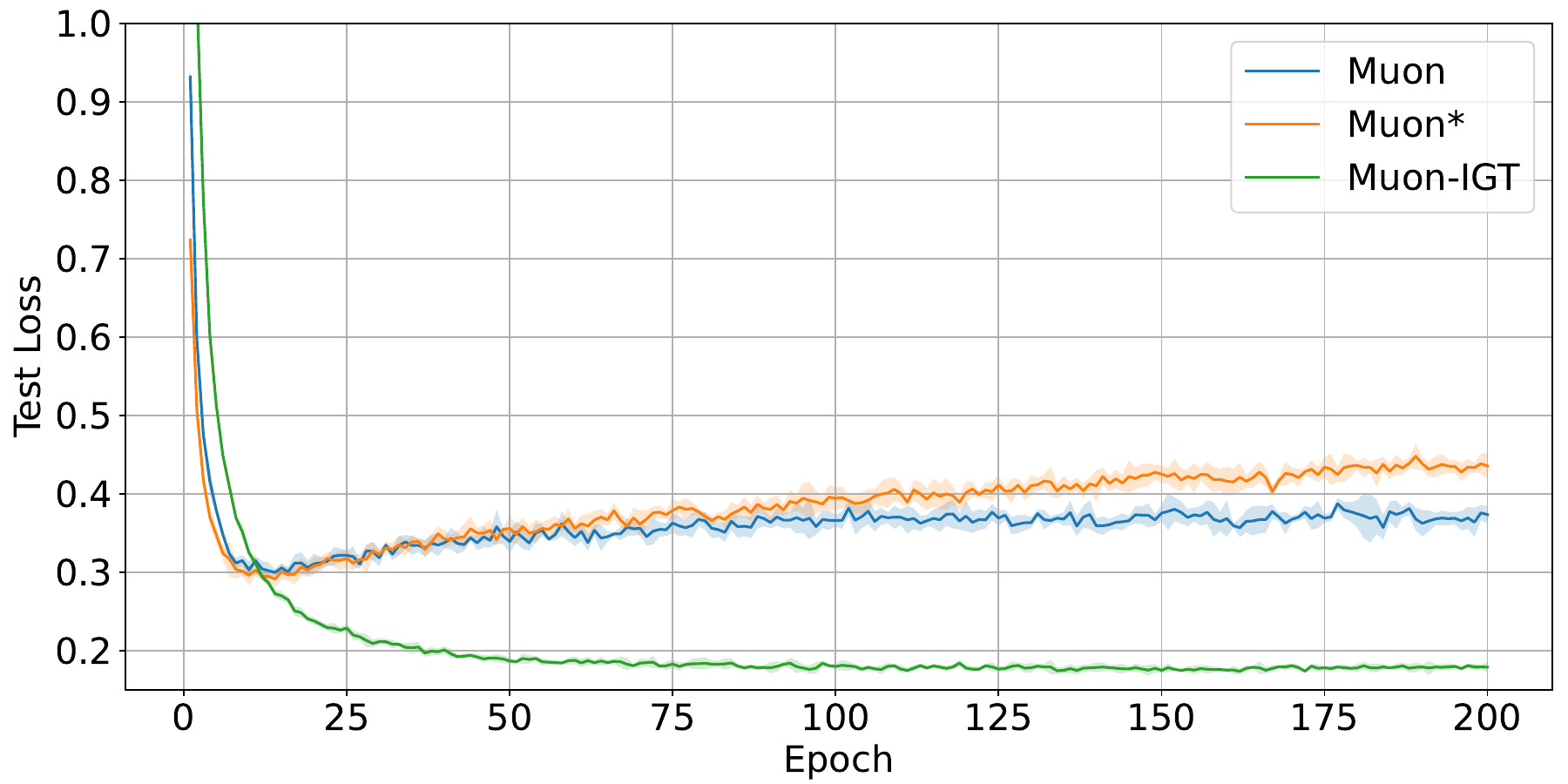}
        \caption{Test loss over epochs.}
        \label{fig:appendix-ablation-loss}
    \end{subfigure}
    \caption{Ablation study for Muon-IGT on CIFAR-10 with ResNet-18.}
    \label{fig:appendix-ablation}
\end{figure}

\subsubsection{Robustness to Learning Rate and Weight Decay}
To assess robustness to hyperparameter choice, we measure the final test accuracy of Muon-IGT over multiple learning-rate and weight-decay combinations. For weight decay, we use the values listed in Table~\ref{tab:img-cls-search}. For the learning rate, after identifying the best-performing value in the initial hyperparameter search, we construct a local grid from nearby powers of two, resulting in \(\{2^{-15},2^{-14},\ldots,2^{-7}\}\). %
We consider two representative values of \(\beta_1\), namely \(0.9\) and \(0.95\), while keeping the remaining settings fixed. Each entry in Figure~\ref{fig:appendix-robustness} reports the final test accuracy obtained by a single pair of learning rate and weight decay.

\begin{figure}[t]
    \centering
    \begin{subfigure}[b]{0.49\textwidth}
        \centering
        \includegraphics[width=\linewidth]{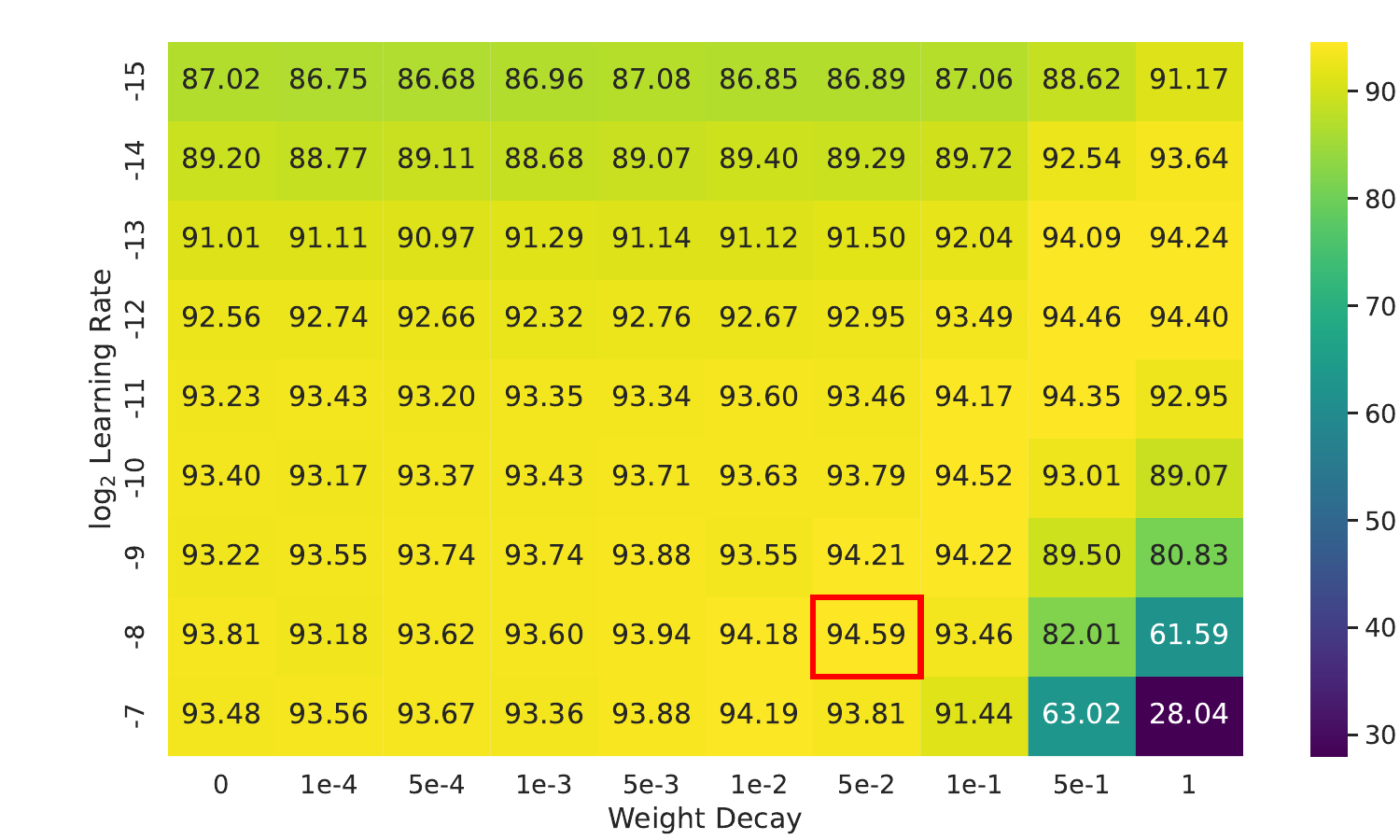}
        \caption{\(\beta_1=0.9\)}
        \label{fig:appendix-robustness-b09}
    \end{subfigure}
    \hfill
    \begin{subfigure}[b]{0.49\textwidth}
        \centering
        \includegraphics[width=\linewidth]{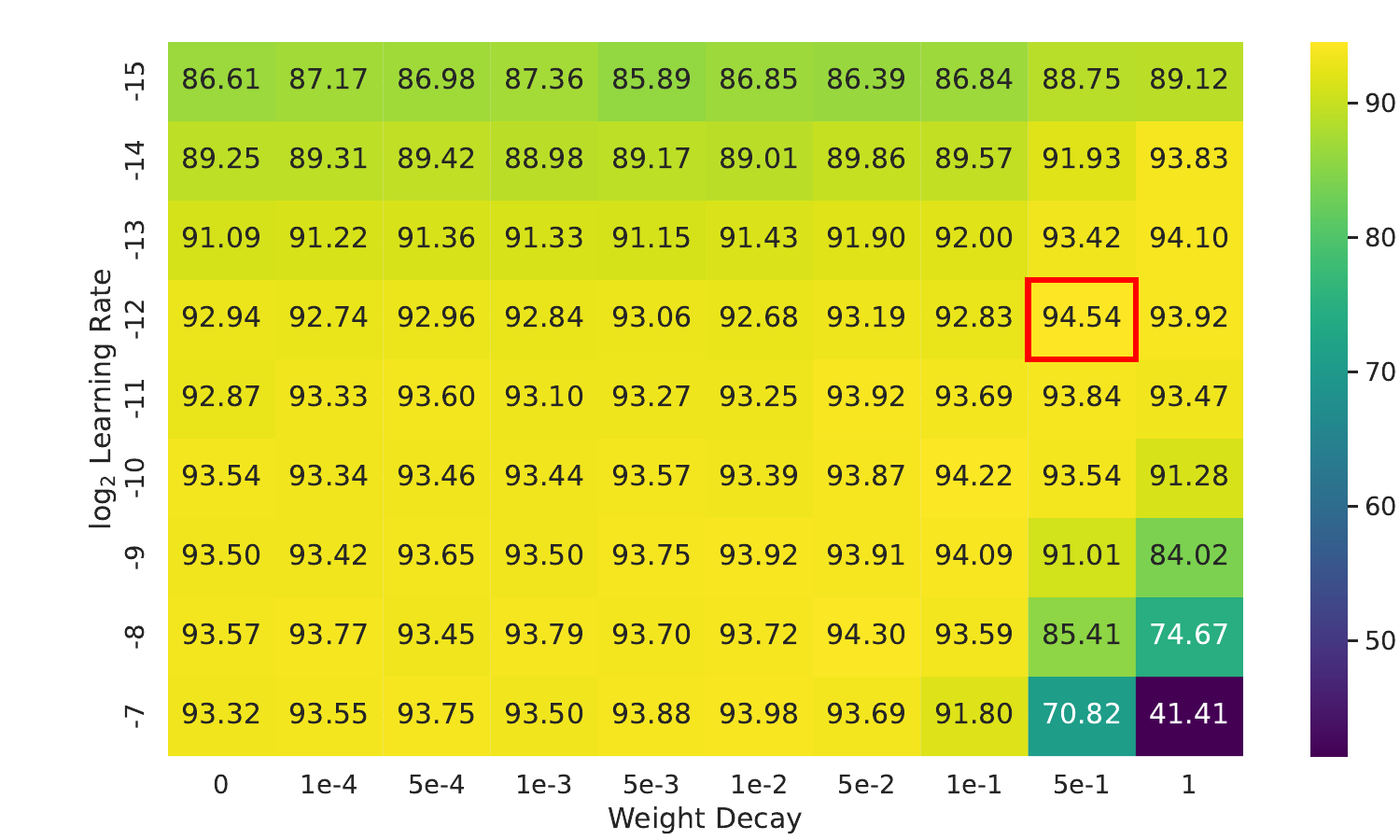}
        \caption{\(\beta_1=0.95\)}
        \label{fig:appendix-robustness-b095}
    \end{subfigure}
    \caption{Final test accuracy of Muon-IGT on CIFAR-10 with ResNet-18 as a function of learning rate and weight decay under two choices of \(\beta_1\).}
    \label{fig:appendix-robustness}
\end{figure}

Figure~\ref{fig:appendix-robustness} shows broad high-accuracy regions for both \(\beta_1=0.9\) and \(\beta_1=0.95\), indicating that Muon-IGT performs well over a wide range of learning rates and weight decays. This suggests that the empirical gain of Muon-IGT is not confined to a narrowly tuned hyperparameter configuration, but remains stable across a broad portion of the search space.

\subsection{Language Modeling}\label{app:lang_model}

{We next evaluate the methods on language modeling using the Shakespeare dataset and nanoGPT. For each optimizer, we select hyperparameters based on the final test loss. Compared with the image-classification setting in Appendix \ref{app:additional-experiments}, some baselines are omitted because their performance is substantially worse.}

As shown in Figure~\ref{fig:appendix-shakespeare-b098}, Muon-IGT outperforms the baselines in both training loss and test loss from an early stage of optimization. This indicates that the advantage of Muon-IGT persists even under a smaller momentum parameter.

\begin{figure}[!htb]
    \centering
    \begin{subfigure}[b]{0.49\textwidth}
        \centering
        \includegraphics[width=\linewidth]{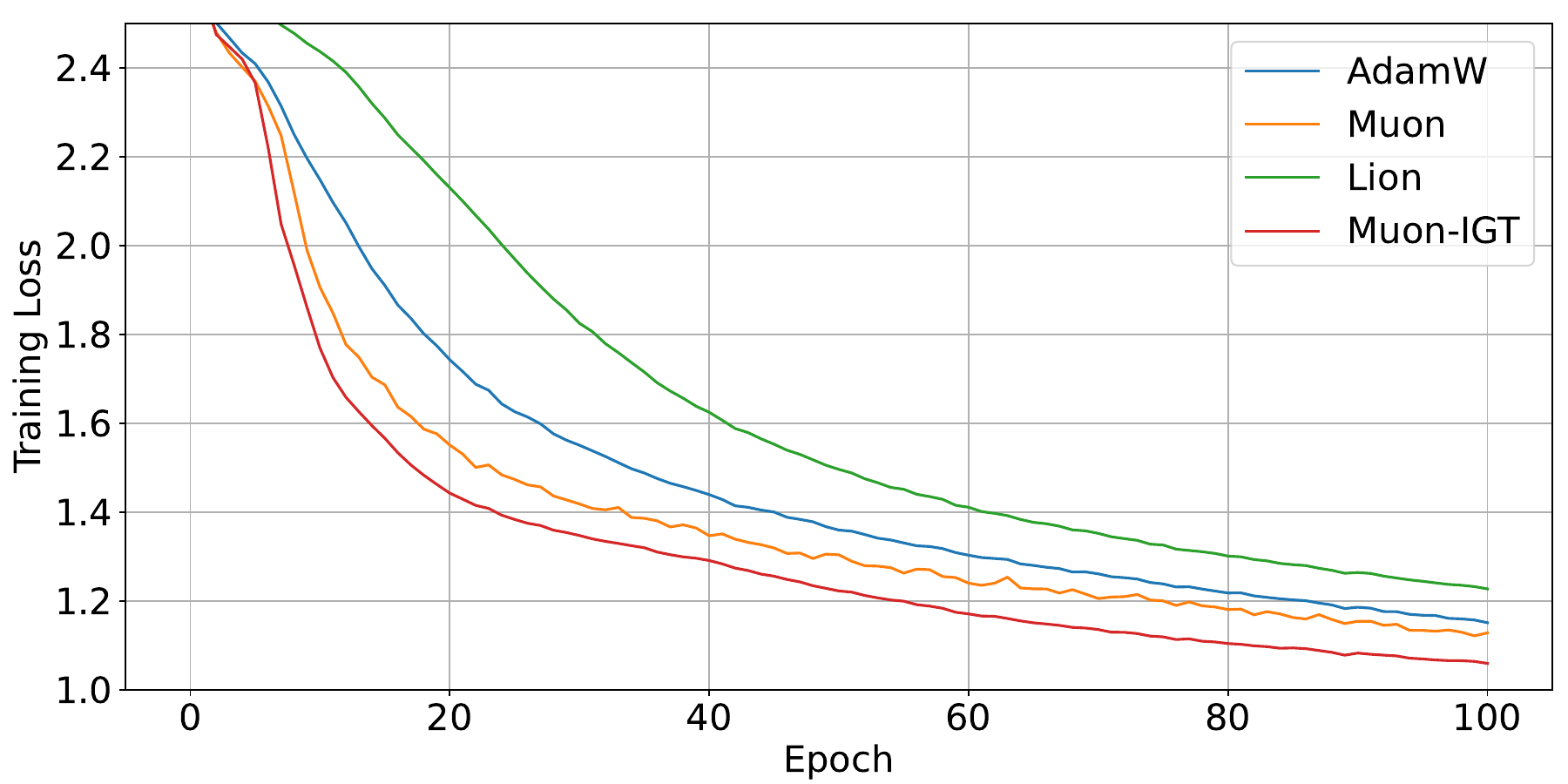}
        \caption{Training loss over steps.}
        \label{fig:appendix-shakespeare-b098-train}
    \end{subfigure}
    \hfill
    \begin{subfigure}[b]{0.49\textwidth}
        \centering
        \includegraphics[width=\linewidth]{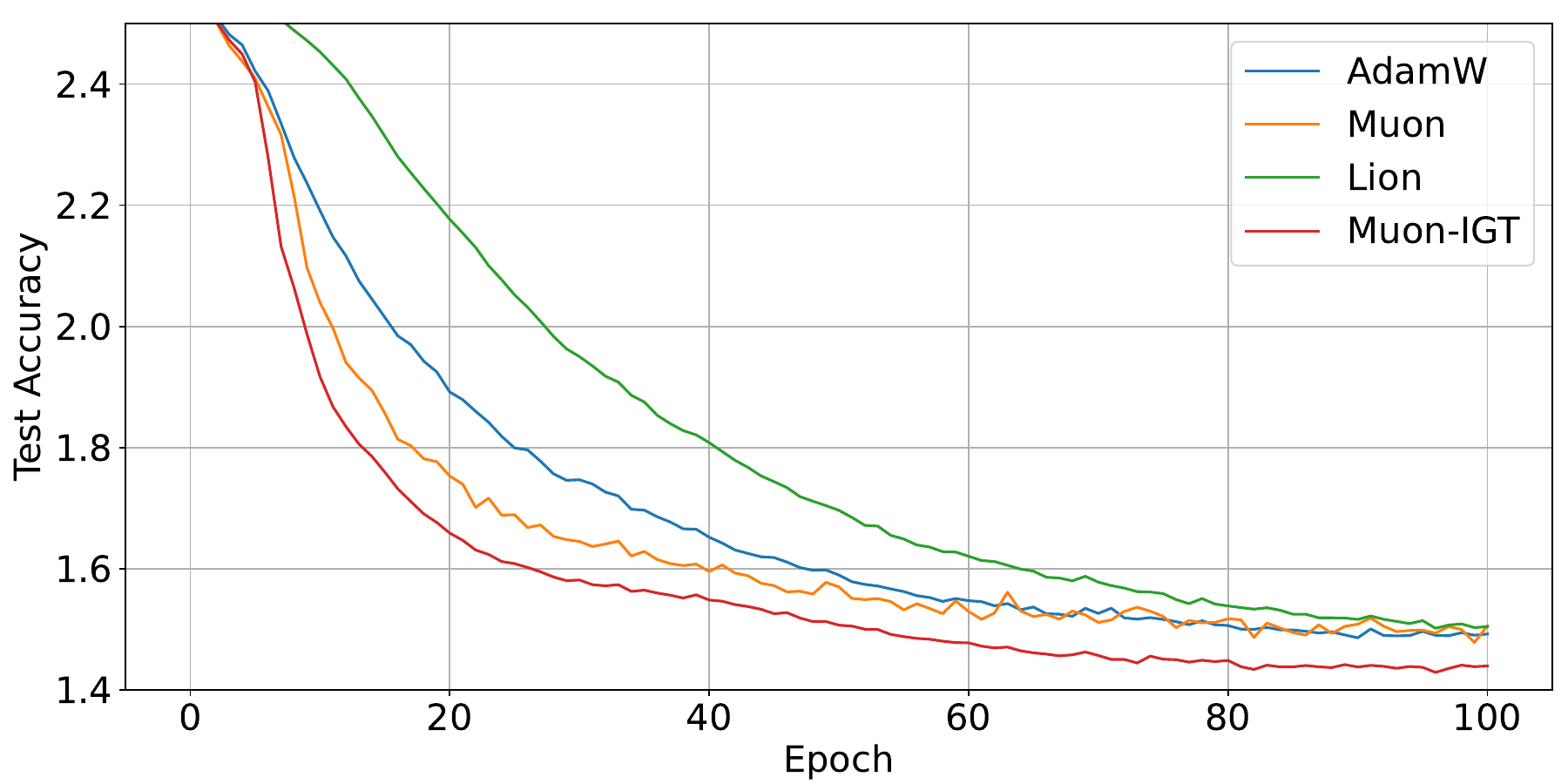}
        \caption{Test loss over steps.}
        \label{fig:appendix-shakespeare-b098-test}
    \end{subfigure}
    \caption{Language modeling results on Shakespeare with nanoGPT (10M)}
    \label{fig:appendix-shakespeare-b098}
\end{figure}

\subsection{Large-scale Language Modeling} \label{subsec:larger}

To examine whether Muon-IGT scales to larger language-modeling tasks, we further evaluate it on OpenWebText using a 124M-parameter nanoGPT model, rather than the 10M model used in Appendix~\ref{app:lang_model}.

\begin{figure}[!htb]
    \centering
    \begin{subfigure}[b]{0.49\textwidth}
        \centering
        \includegraphics[width=\linewidth]{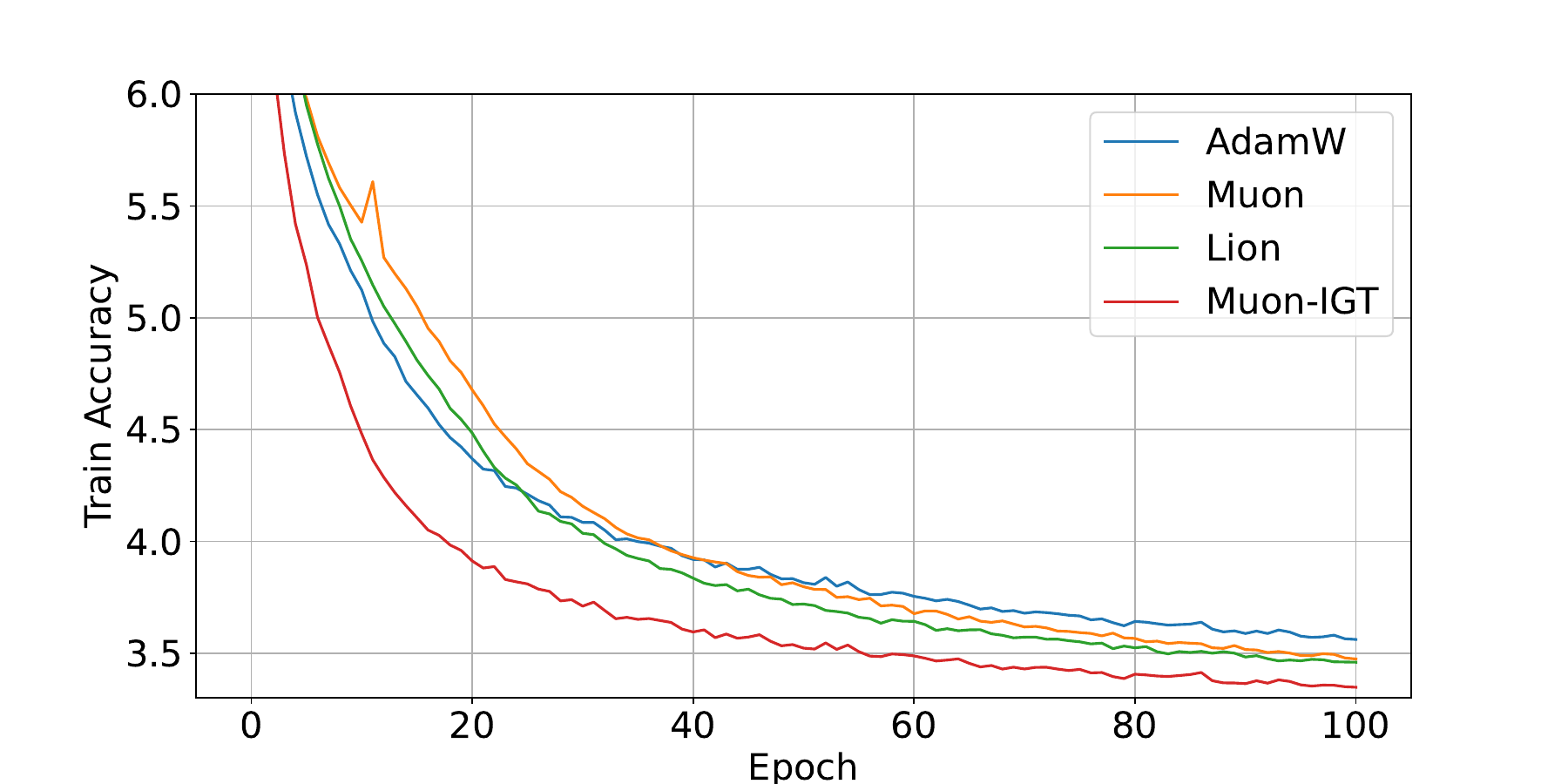}
        \caption{Training loss over steps.}
        \label{fig:appendix-owt-train}
    \end{subfigure}
    \hfill
    \begin{subfigure}[b]{0.49\textwidth}
        \centering
        \includegraphics[width=\linewidth]{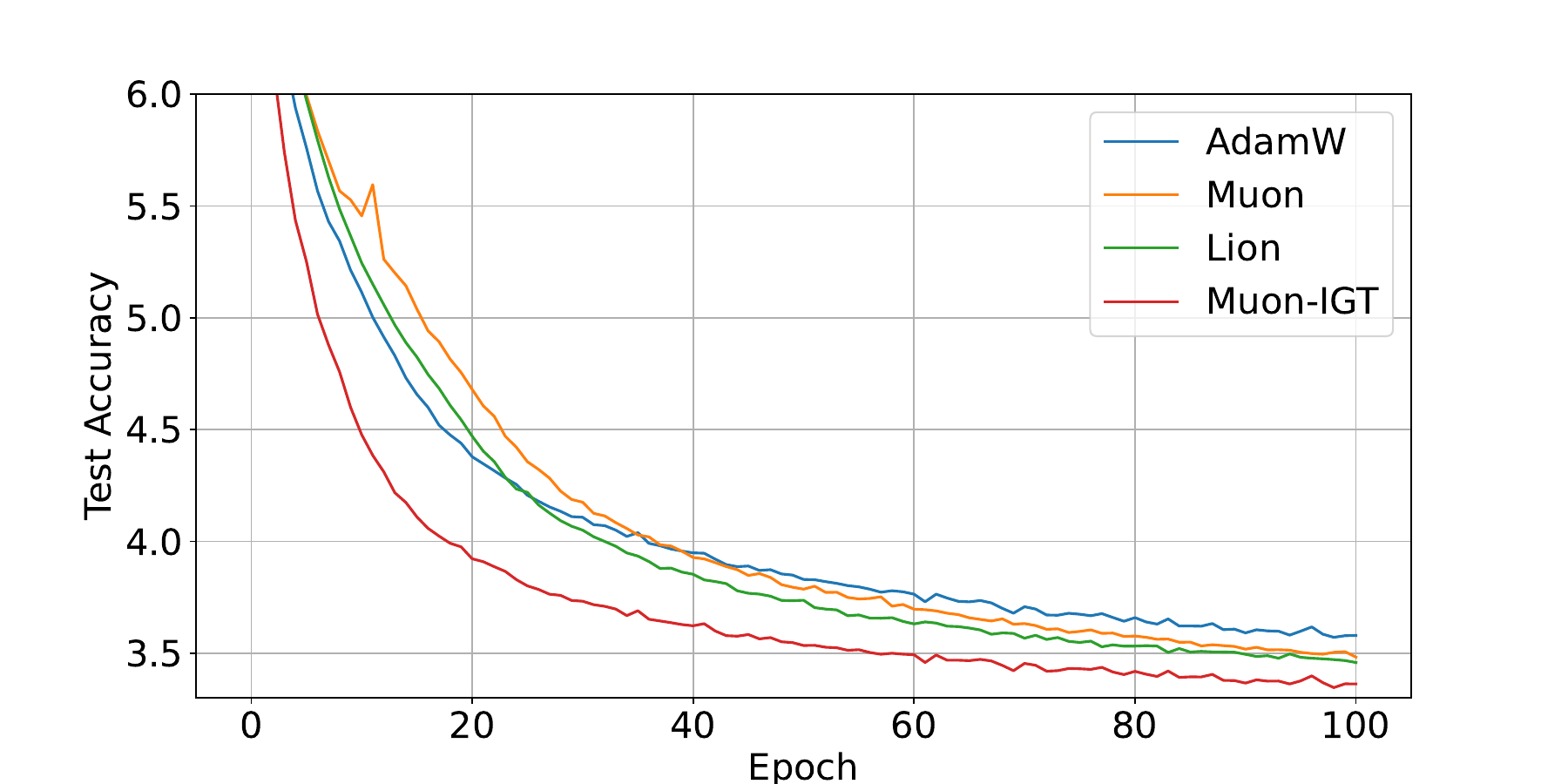}
        \caption{Test loss over steps.}
        \label{fig:appendix-owt-test}
    \end{subfigure}
    \caption{Large-scale language modeling results on OpenWebText with nanoGPT (124M).}
    \label{fig:appendix-owt}
\end{figure}

We observe the same qualitative trend as in the smaller-scale language-modeling experiment. Muon-IGT outperforms the baselines in both training loss and test loss from the early stage of training, suggesting that the method scales favorably to larger language-modeling problems.

\section{Limitations}
\label{sec:limitations}
Our results should be interpreted within the scope of the present theoretical setting. The convergence guarantees for stochastic LMO and LMO-VR rely on standard smoothness and stochastic-noise assumptions, while the improved guarantee for LMO-IGT further leverages second-order smoothness through the transported-gradient construction. Although these assumptions are standard in stochastic optimization, they remain idealized relative to modern large-scale deep learning practice. Accordingly, our theory is best viewed as a principled rate comparison within a canonical analytical regime, rather than a complete characterization of optimizer behavior in practical training.

Another limitation is that the analysis focuses on expected first-order stationarity measured by the regularized support function. We do not establish stronger guarantees such as high-probability bounds, general last-iterate convergence, or a unified treatment of approximate or biased LMOs. Extending the framework in these directions would further clarify the robustness and scope of the RSF perspective.

These limitations do not diminish the main contribution of the paper, namely that transported gradients extend beyond Euclidean normalization to general LMO-based updates under a unified stationarity framework. Rather, these limitations delineate the scope of our theoretical claims and highlight the settings in which our unified perspective is most informative.

\section{Broader Impact}
\label{sec:broader-impact}
This work studies general-purpose optimization methods for training machine learning models. Its primary impact is methodological: a better understanding of LMO-based optimizers may lead to more compute- and memory-efficient training, and to more consistent comparisons across constrained and unconstrained formulations. If such gains persist at scale, they could reduce experimental cost and broaden access to model training. However, our results do not support strong claims about net energy savings in realistic deployments.

As with most advances in optimization, improved training efficiency may lower the cost of building more capable models, potentially amplifying both beneficial applications and undesirable uses. Since our contribution is foundational, its broader societal impact depends on how downstream models are developed and deployed.

Finally, by emphasizing unified stationarity measures together with explicit consideration of computational overhead, our work encourages more transparent and consistent evaluation of optimization methods.

\end{document}